\documentclass[12pt,letterpaper]{article}
\usepackage[a4paper, total={7in, 10in}]{geometry}

\usepackage{helvet}
\usepackage{authblk}

\usepackage{amsfonts}
\usepackage{amsmath}
\usepackage{amsthm}
\usepackage{amssymb}
\usepackage{bm}
\usepackage[dvipdfmx]{color}
\usepackage[dvipdfmx]{graphicx}
\usepackage{url}
\usepackage{subcaption}
\usepackage{algorithm}
\usepackage{algorithmic}
\usepackage{Sty/mediabb}
\usepackage{natbib}
\usepackage{mathtools}
\usepackage{svg}
\usepackage{nameref}

\usepackage{Sty/mcr}
\newcommand{\refsec}[1]{Section ``\nameref{#1}''}
\newcommand{\refsecs}[2]{Sections ``\nameref{#1}'' and ``\nameref{#2}''}

\makeatletter
\renewcommand{\maketitle}{\bgroup\setlength{\parindent}{0pt}
\begin{flushleft}
  \textbf{\@title}
  
  \@author
\end{flushleft}\egroup}
\makeatother

\title{Efficient Model Selection for Predictive Pattern Mining Model by Safe Pattern Pruning}
\date{}

\author[1]{Takumi Yoshida}
\author[2,*]{Hiroyuki Hanada}
\author[1]{Kazuya Nakagawa}
\author[3]{Kouichi Taji}
\author[2,4]{Koji Tsuda}
\author[2,3,*]{Ichiro Takeuchi}

\affil[1]{Nagoya Institute of Technology, Nagoya, Aichi, Japan}
\affil[2]{RIKEN, Wako, Saitama, Japan}
\affil[3]{Nagoya University, Nagoya, Aichi, Japan}
\affil[4]{The University of Tokyo, Bunkyo-ku, Tokyo, Japan}
\affil[*]{Co-correspondence: hiroyuki.hanada@riken.jp, ichiro.takeuchi@mae.nagoya-u.ac.jp}

\usepackage{natbib}

\begin{document}

\maketitle

\section*{Summary}

Predictive pattern mining is an approach used to construct prediction models when the input is represented by structured data, such as sets, graphs, and sequences.
The main idea behind predictive pattern mining is to build a prediction model by considering substructures, such as subsets, subgraphs, and subsequences (referred to as patterns), present in the structured data as features of the model.
The primary challenge in predictive pattern mining lies in the exponential growth of the number of patterns with the complexity of the structured data.
In this study, we propose the Safe Pattern Pruning (SPP) method to address the explosion of pattern numbers in predictive pattern mining.
We also discuss how it can be effectively employed throughout the entire model building process in practical data analysis.
To demonstrate the effectiveness of the proposed method, we conduct numerical experiments on regression and classification problems involving sets, graphs, and sequences.
%
%(148 words)

\section*{Keywords}

Predictive pattern mining, Itemset mining, Sequence mining, Graph mining, Sparse learning, Safe screening, Convex optimization

\section*{Introduction} \label{sec:introduction}

%\section{Introduction}

In various practical problems, it is necessary to handle \emph{structure data} such as \emph{sets}, \emph{graphs}, and \emph{sequences}.
For example, in the field of life sciences, interactions between different genes are represented as sets, drug-able chemical compounds are represented as graphs, and amino acid sequences that make up proteins are represented as sequence data.
In this paper, we consider prediction problems such as regression and classification when the input is structure data.
In the aforementioned cases, problems such as predicting the presence or absence of a disease based on interactions between genes, predicting the effectiveness of drugs based on chemical compound structures, and predicting allergic reactions based on the amino acid sequences of food proteins are examples of our target applications.
In predictive modeling for structure data, the challenge is how to represent the structure data so that they can be fed into machine learning framework.
In this paper, we consider a class of machine learning models called \emph{predictive pattern mining}~\citep{saigo2007mining,saigo2009linear,saigo2009gboost,duverle2013discovering,suzumura2017selective,takayanagi2018entire,yoshida2018safe,le2018whinter,yoshida2019learning,le2020stat,hazimeh2020learning,bunker2021supervised,das2022fast,kato2022safe}.

There are mainly three types of machine learning approaches for structure data.
The first approach is the kernel-based approach. 
In this approach, a kernel function that can quantify similarities between sets, graphs, sequences, and other structures are introduced, and it is used together with kernel-based machine learning methods such as Support Vector Machines and Gaussian Process.
A variety of kernel functions specialized for each type of structures have been proposed and used in practical problems~\citep{gartner2003graph,kashima2003marginalized,vishwanathan2010graph,shervashidze2009efficient,tsuda2002marginalized,lodhi2002text,leslie2001spectrum}.
The second approach is deep learning-based approach.
This approach employs neural network models with special types of input and hidden layers that are designed to handle the inputs in the form of sets, graphs, sequences, and other structures.
For example, PointNet is a neural network especially developed for set data~\citep{qi2017pointnet,lee2019set}, graph neural networks are used for graph data~\citep{scarselli2008graph,kipf2016semi,velivckovic2017graph,zhou2020graph,xu2018powerful}, and there are many neural network architectures for sequence data such as recurrent neural networks and LSTM~\citep{hochreiter1997long,graves2013generating,sutskever2014sequence,vaswani2017attention}.
These two approaches are effective when we are only interested in prediction.
In many practical problems, however, simply having good predictive performance is not enough.
In prediction modeling for structure data, knowledge extraction such as identifying the sub-structures that contribute to the prediction is required for the explanation and interpretation.

The third approach is predictive pattern mining, which is the subject of this paper.
In contrast to the two aforementioned approaches, knowledge extraction is possible in predictive pattern mining approach.
A common feature among many types of structure data such as sets, graphs, and sequences is that they can be decomposed into sub-structures.
For example, considering a set of three genes $\{g_A, g_B, g_C\}$ as a set data, it has the following sub-structures:
\begin{align*}
\emptyset,
\{g_A\},
\{g_B\},
\{g_C\},
\{g_A, g_B\},
\{g_A, g_C\},
\{g_B, g_C\},
\{g_A, g_B, g_C\},
\end{align*}
where
$\emptyset$
is the empty set. 
If we consider a predictive model that takes a set as an input, it is possible to extract knowledge by knowing which sub-structures contribute significantly to the prediction.
In this paper, these substructures are called \emph{patterns}.
Fig.~\ref{fig:pattern_examples} shows examples of patterns for set, graph, and sequence data.

The difficulty in predictive pattern mining lies in the computational complexity of efficiently handling exponentially increasing number of patterns.
In any of the set, graph, or sequence structure data discussed so far, the number of all possible sub-structures (patterns) are huge, making it difficult to consider a machine learning model that naively treats all possible patterns as features.
Since the patterns that affect predictions are often only a small part of the vast number of pattern features, the basic strategy for predictive pattern mining is to efficiently identify relevant pattern features and remove irrelevant pattern features that do not affect predictions.
In the field of pattern mining, algorithms utilizing the fact that patterns can be represented in a tree structure have been exploited for tasks such as enumerating frequently occurring patterns (Fig. \ref{fig:tree_pruning}).
Our main contribution in this paper is to propose a method for efficiently finding patterns that significantly contribute to predictions by using the tree representation of patterns, similar to other pattern mining methods.

To this end, we propose a novel technique called \emph{safe pattern pruning (SPP)} by combining \emph{safe screening}, which has been developed in the field of sparse modeling, and pattern mining that utilizes the tree-based pattern representations.
In the SPP method, we consider the sparse estimation of linear models that can have any pattern as a feature for predictive pattern mining model, and identify a specific set of patterns where the coefficients become zero in the optimal solution.

This paper is an extended version of the preliminary conference proceeding presented by a part of the authors~\citep{nakagawa2016safe}. The entire content of \refsec{sec:sec4} is a novel addition to this paper, and a significant portion of the experiments described in \refsec{sec:experiments} were newly conducted in this paper.
In this paper, we newly introduce methods to effectively perform model selection by effectively using the SPP, provide a software that can comprehensively handle regression and classification problems with sets, graphs, and sequences, and present new numerical experimental results.

\section*{Results} 

%\section*{Preliminaries}
%
%In this section, we formulate the predictive pattern mining problem, and introduce sparse learning and safe screening method which can be used for efficiently removing patterns that do not contribute predictions.

\subsection*{Notations}
\label{sec:notation-outline}
We use the following notations in the rest of the paper. 
For any natural number $n$, we define $[n] \coloneqq \{1, \ldots, n\}$. 
For a matrix $Z\in\mathbb R^{n\times m}$, its $i$-th row and $j$-th column are denoted as $\bm z_{i:}$ and $\bm z_{:j}$, respectively, for $i \in [n]$ and $j \in [m]$. 
The $L_1$ norm, the $L_2$ norm of a vector $v \in \RR^n$ are defined as $\|v\|_1 \coloneqq \sum_{i \in [n]} |v_i|$ and $\|v\|_2 \coloneqq \sqrt{\sum_{i \in [n]} |v_i|^2}$, respectively.
In addition, the symbols $\bm 0$ and $\bm 1$ denote vectors of suitable dimensionality, where all of their elements are equal to 0 and 1, respectively.

\subsection*{Problem Setup: Predictive Pattern Mining}
Let $\mathcal S$ be the input space for structure data, e.g., $\mathcal S$ is a collection of sets, graphs, and, sequences.
We assume that a certain partial order $\sqsubset$ on $\mathcal S$ is defined, which represents an inclusion relationship.
For example, in the case of subset mining, we can use well-known subset operator $\subset$ as $\sqsubset$.
Let $\mathcal Y$ be the space of response (i.e., output space).
We consider $\mathcal Y \subseteq \mathbb R$ for regression problems and $\mathcal Y = \{-1, +1\}$ for binary classification problems.

We denote the training dataset with $n$ instances as $\mathcal D = \{(S_i, y_i)\}_{i\in[n]}$, where $S_i \in \cS$ and $y_i \in \cY$ are the pair of structure input and response for the $i$-th training instance. 
Let $\cP$ be the set of all sub-structures (patterns) contained in $\{S_i\}_{i \in [n]}$, and $p_1, p_2, \ldots, p_d \in \cP$ be the elements of $\cP$. 
Here, $d$ is the total number of patterns, which increases exponentially with the complexity of structure inputs. 
In predictive pattern mining, we consider a generalized linear model in the form of
\begin{align}
 \label{eq:linear-model}
 g(y_i) = \beta_0 + \beta_1 x_{i1} + \beta_2 x_{i2} + \cdots + \beta_d x_{id}, i \in [n], 
\end{align}
where $x_{ij}$ equals 1 if the $i$-th input structure has the $j$-th pattern and 0 otherwise, $\beta_0$ and $\{\beta_j\}_{j \in [d]} \in \RR$ are the model coefficient, and $g$ is the link function for generalized linear model.
With this notation, the training set for the predictive pattern mining is also represented as $(X, \bm y)$ where $X \in \{0, 1\}^{n \times d}$ is the matrix with $x_{ij}$ being the $(i, j)$-th element and $\bm y \in \RR^n$ is the vector with $y_i$ being the $i$-th response.
Furthermore, we denote the $i$-th row of $X$ as $\bm x_i$ and a vector of coefficients as $\bm \beta = [\beta_1, \ldots, \beta_d] \in \RR^d$. 
Then, the linear model in Eq.\eq{eq:linear-model} is simply written as $\bm y = X \bm \beta + \bm 1 \beta_0$ or $y_i = \bm x_i^\top \bm \beta + \beta_0$, $i \in [n]$. 

The goal of predictive pattern mining is to find the vector of coefficients parameters $(\bm\beta, \beta_0) \in \RR^d \times \RR$ that minimizes the following class of loss function: 
\begin{equation}
 \begin{split}
  \bm\beta^* & \coloneqq \argmin_{\bm\beta\in\mathbb R^d, \beta_0\in\mathbb R} P(\bm\beta, \beta_0),\\
  P(\bm\beta, \beta_0) & \coloneqq L_{\bm y}(X\bm\beta + \bm 1 \beta_0) + \Omega_{(\lambda, \kappa)}(\bm\beta),
 \end{split}\label{eq:primal_problem}
\end{equation}
where $L_{\bm y}:\mathbb R^n\rightarrow\mathbb R$ is a convex loss function with Lipschitz continuous gradient, and $\Omega_{(\lambda, \kappa)}$ is a convex regularization function.
In this paper, we focus on the following regularization function as $\Omega_{(\lambda, \kappa)}$, which is called Elastic Net regularization~\citep{zou2005regularization}: 
\begin{equation}
 \label{eq:elasticnet-regularization}
 \Omega_{(\lambda, \kappa)}(\bm\beta) = \lambda \left(\|\bm\beta\|_1 + \frac{\kappa}{2}\|\bm\beta\|_2^2\right),
\end{equation}
where $\lambda > 0$ and $\kappa \ge 0$ are the hyper-parameters for tuning the strength of regularization.
When using this regularization function, a sparse solution is obtained, meaning that many coefficients $\beta^*_j$ shrink to zero in the optimal solution.

The dual problem of \eqref{eq:primal_problem} is written as
\begin{equation}
 \begin{split}
  \bm\alpha^* &\coloneqq\max_{\bm\alpha\in\mathbb R^n,~\bm\alpha^\top\bm 1 = 0} D(\bm\alpha),\\
  D(\bm\alpha) &\coloneqq -L^*(- \bm\alpha) - \Omega^*_{(\lambda, \kappa)}(X^\top\bm\alpha),
 \end{split}\label{eq:dual_problem}
\end{equation}
where $f^*$ indicates the convex conjugate of a convex function $f$, which is defined as follows.
\begin{defi}[convex conjugate]
    Let $f:\mathbb R^n\rightarrow\mathbb R$ be a convex function, the convex conjugate $f^*: \mathbb R^n\rightarrow\mathbb R$ is defined as 
    \[
        f^*(\bm v) = \sup_{\bm u\in\mathbb R^n}\left\{ \bm u^\top\bm v - f(\bm u) \right\}.
    \]
\end{defi}
We refer to the optimization problem in \eqref{eq:primal_problem} as the primal problem.
In this paper, we efficiently solve the predictive pattern mining problem by effectively combining the primal and the dual problems.

\subsection*{Sparse Learning and Safe Screening}

A class of methods for obtaining sparse solutions by using a sparsity-inducing regularization term such as \eq{eq:elasticnet-regularization} is called \emph{sparse learning}.
In sparse learning, the set of features whose optimal solution is non-zero is called the active set, and is denoted by 
\begin{align*}
 \cA^* := \left\{ j \in [d] \mid \beta^*_j \neq 0 \right\}
\end{align*}
One characteristic of sparse learning is that the optimal solution for a dataset containing the features in the active set gives the same optimal solution obtained for a dataset containing all features.
Concretely, let us consider a superset $\cA \supseteq \cA^*$ of the active set and a dataset $(X_{\cA}, \bm y)$ containing only the features belonging to $\cA$.
Then, the optimal solution for this dataset 
\begin{align*}
 \bm \beta^*(\cA) := \argmin_{\bm \beta \in \RR^{|\cA|}, \beta_0\in\mathbb R}  L_{\bm y}(X_\cA \bm \beta + \bm 1 \beta_0) + \Omega_{(\lambda, \kappa)}(\bm \beta)
\end{align*}
has a property that 
\begin{align*}
 \beta^*_j(\cA) = \beta^*_j ~ \forall j \in \cA
\end{align*}
This property implies that, if we can obtain a superset that contains the active set, it is sufficient to solve the optimization problem for a smaller dataset with smaller number of features. 

In general, active set cannot be obtained until the optimization problem is solved.
However, by using an approach called \emph{safe screening}, there is a case where it is possible to identify features that cannot be active in the optimal solution, i.e., features for which $\beta^*_j=0$, before solving the optimization problem.
Specifically, by using the convex optimization theory, it can be shown that there is a relationship between the optimal solutions of the primal problem \eq{eq:primal_problem} and the dual problem \eq{eq:dual_problem}, expressed as
\begin{equation}
 |\Xj^\top\bm\alpha^*| < \lambda \Longrightarrow \beta^*_j = 0, \quad \forall j\in[d].
 \label{eq:optimality}
\end{equation}
The basic idea of safe screening is to compute an upper bound on $|\Xj^\top\bm\alpha^*|$ in \eq{eq:optimality}. 
If the upper bound is smaller than $\lambda$, the condition in \eqref{eq:optimality} is satisfied, meaning that the optimal solution of the corresponding primal problem is $\beta^*_j=0$ and this feature can be removed beforehand.
Safe screening was proposed by \citep{ghaoui2010safe}, and since then, improvements to the method have been made \citep{wang2013lasso,fercoq2015mind,ndiaye2017gap,xiang2016screening,bonnefoy2015dynamic}, and its range of applications has been expanded~\citep{ogawa2013safe,shibagaki2015regularization,shibagaki2016simultaneous,hanada2018efficiently,okumura2015quick,ndiaye2019safe,takada2016secure}.

Among several options, we employ a safe screening method called \emph{GAP safe screening}~\citep{fercoq2015mind,ndiaye2017gap}. 
The basic idea of GAP safe screening is to use an arbitrary primal feasible solution $\tilde{\bm \beta} \in \RR^d$ and an arbitrary dual feasible solution $\tilde{\bm \alpha} \in \RR^n$ to compute an upper bound on $|\Xj^\top\bm\alpha^*|$.
The following lemma states that, given a pair of primal and dual feasible solutions $(\tilde{\bm \beta}, \tilde{\bm \alpha})$, it is possible to determine the range of the dual optimal solution $\bm \alpha^*$. 
\begin{lemm}[GAP Safe Screening Rule]
 \label{lem:safe_screening_rule}
 Suppose that $\nabla L$ is Lipschitz continuous with constant $\gamma > 0$.
 For any pair of feasible solution $(\bzero, \azero)$ and $j\in[d]$, let us define, what is called, \emph{safe screening score} as follows:
 \begin{align}
  \label{eq:safe_screening_score}
  u_j(\bzero, \azero)\coloneqq |\Xj^\top\bm\azero| + r(\bzero, \azero)\|\Xj - \Pi_{\bm 1}(\Xj)\|_2 . 
 \end{align}
 Then, 
 \begin{align}
  \label{eq:safe_screening_rule}
  u_j(\bzero, \azero) < \lambda ~\Rightarrow~ \beta_j^* = 0,
 \end{align}
 where
 \begin{align*}
  & r(\bzero, \azero) \coloneqq \sqrt{2\gamma(P(\bzero) - D(\azero))}, \\
  & \Pi_{\bm u}(\bm v) \coloneqq \frac{\bm u^\top\bm v}{\bm u^\top\bm u}\bm u,\;\bm u, \bm v\in\mathbb R^n.
 \end{align*}
 ($\Pi_{\bm u}(\bm v)$ is known as the {\em projection} of $\bm v$ onto $\bm u$.)
\end{lemm}
The proof of Lemma~\ref{lem:safe_screening_rule} is presented in Note S1.
This lemma indicates that, given a pair of primal and dual feasible solutions $(\tilde{\bm \beta}, \tilde{\bm \alpha})$, we can first compute the upper bound
$u_j(\bzero, \azero)$
for each
$j \in [d]$,
and remove the feature
if 
$u_j(\bzero, \azero) < \lambda$.

Our basic idea is to apply this GAP safe screening to predictive pattern mining.
However, since the number of all possible features $d$ is exponentially increasing, it is impossible to compute the upper bound $u_j(\bzero, \azero)$ for each pattern.
To address this challenge, in the next section, we extend the safe screening rule so that it can identify a group of removable patterns at once.

\subsection*{Safe Pattern Pruning (SPP)}
\label{sec:proposed_method}
The basic idea of SPP is to represent the relationship among patterns in a tree (Fig. \ref{fig:tree_pruning}) and identify a group of patterns for which the optimal coefficients satisfy $\beta^*_j=0$ by pruning the tree. 
To obtain the pruning rule, we exploit monotonicity of patterns, i.e., the occurrence of patterns decreases monotonically as pattern grows in the tree. 
We describe this property in the following lemma more specifically.
\begin{lemm}[Monotonicity of patterns]
 \label{lem:monotnicity}
 Let $p_j, p_k\in\mathcal S$ be patterns in $\{S_i\}_{i \in [n]}$ such that $p_k \sqsubset p_j$.
 Then, for any $i\in[n]$,
 \[
 x_{ik} = 1 ~\Rightarrow~ x_{ij} = 1.
 \]
\end{lemm}
This is obvious because if an input instance $S_i$ has $p_k$ (i.e., $p_k\sqsubset S_i$), then $p_k \sqsubset p_j \sqsubset S_i$ also holds.
Using this lemma, we derive the following theorem.
\begin{theo}[Safe Pattern Pruning (SPP) Rule]
 \label{the:safe_pattern_pruning_rule}
 For any pair of feasible solution $(\bzero, \azero)$ and $j\in[d]$, let us define, what we call, \emph{safe pattern pruning score (SPP score)} as follows:
 \begin{align}
  \label{eq:spp_score}
  v_j(\bzero, \azero)\coloneqq \max\left\{ 
  \sum_{i:\tilde{\alpha}_i > 0}x_{ij}\tilde{\alpha}_i,
  -\sum_{i:\tilde{\alpha}_i < 0}x_{ij}\tilde{\alpha}_i\right\}
  + r(\bzero, \azero)\|\bm x_{:j}\|_2. 
 \end{align}
 Then, 
 \begin{align*}
  v_j(\bzero, \azero) < \lambda
  ~\Rightarrow~
  \beta^*_k = 0
  ~\forall k \in [d] \text{ s.t. } p_k \sqsubset p_j. 
 \end{align*}
\end{theo}

The proof is presented in Note S2.
Using Theorem~\ref{the:safe_pattern_pruning_rule}, it is possible to screen a group of patterns at once during the process of searching in the tree that represents the relationships between patterns.
Specifically, when searching for the screen-able patterns in the tree from the root node to descent nodes, if the SPP score of a pattern $p_j$ corresponding to a certain node of the tree satisfies $v_j(\tilde{\bm \beta}, \tilde{\bm \alpha}) < \lambda$, all patterns $p_k$ corresponding to its descendant nodes satisfies $p_k \sqsubset p_j$, so they can be screened out as unnecessary patterns.
Algorithm \ref{alg:safe_pattern_pruning} in Note S4 shows the pseudo-code of the SPP.

To apply the SPP to actual predictive pattern mining, a pair of feasible solutions $(\tilde{\bm \beta}, \tilde{\bm \alpha})$ for the primal and dual problems is necessary.
Although the SPP rules hold for any feasible solutions $(\tilde{\bm \beta}, \tilde{\bm \alpha})$, the tightness of the bound depends on the choice of feasible solutions.
Specifically, because the tightness of the SPP bound is determined by the duality gap $P(\tilde{\bm \beta}) - D(\tilde{\bm \alpha})$ of the feasible solution, the ``closer'' the pair of feasible solutions $(\tilde{\bm \beta}, \tilde{\bm \alpha})$ is to the (unknown) pair of optimal solutions $(\bm \beta^*, \bm \alpha^*)$, the tighter the SPP bound will be.
Therefore, when applying the SPP to actual predictive pattern mining, it is important to obtain feasible solutions that are sufficiently ``close'' to the optimal solutions.

\subsection*{SPP for Model Selection}
\label{sec:sec4}

As mentioned in the previous section, in order to screen out features using the SPP, a feasible pair of solutions $(\tilde{\bm \beta}, \tilde{\bm \alpha})$ for the primal and dual problems, respectively, that are sufficiently close to the optimal solutions $(\bm \beta^*, \bm \alpha^*)$ is necessary.
In practical data analysis, it is often necessary to learn multiple models rather than just obtaining a single predictive pattern mining model, e.g. in selecting hyperparameters $\lambda, \kappa$, or evaluating the generalization performance through cross-validation (CV).
In this section, we discuss how to apply the SPP for a series of model fittings in model selection.
Our main idea is to use the optimal solutions of models fitted in slightly different problem settings (e.g., with similar hyperparameter values or with only a part of the data being different in CV) as reference feasible solutions for the SPP.
Furthermore, we propose an extension of the SPP that enables more effective utilization of multiple reference feasible solutions in practical model fitting scenarios where multiple reference feasible solutions are naturally available.

First, in \refsec{sec:ext_multi_screening}, we describe an extension of safe screening using two different reference feasible solutions.
Next, in \refsec{sec:ext_dynamic}, we introduce an approach called dynamic screening in which the solutions obtained during learning process are used as reference feasible solutions for the SPP.
Furthermore, in \refsec{sec:ext_multi_hyperparameter}, we consider how to apply the SPP in model selection process in which two hyperparameters $\lambda$ and $\kappa$ are optimized.
Finally, in \refsec{sec:ext_cross_validation}, we discuss how to apply the SPP when selecting hyperparameters using cross-validation.
%
%In the rest of the paper, for simplicity of notations, we re-parametrize the Elastic Net penalty using two hyperparameters $\lambda \in \RR$ and $\kappa \in [0, 1]$ as follows:
%\begin{align*}
% \Omega_{(\lambda, \kappa)}(\bm \beta) = \lambda \left(\|\bm \beta\|_1 + \frac{\kappa}{2} \|\bm \beta\|_2^2 \right).
%\end{align*}

\subsubsection*{SPP with Multiple Pairs of Feasible Solutions}
\label{sec:ext_multi_screening}

Let us consider the case where two feasible solutions 
$R_1 \coloneqq (\bone, \tilde{\beta}_0^{(1)}, \aone)$,
$R_2 \coloneqq (\btwo, \tilde{\beta}_0^{(2)}, \atwo)$
are available.
From Lemma 3,
the optimal solution
$\bm\alpha^*$
must be contained in both of the two hyperspheres
$B_1 \coloneqq B(R_1)$
and
$B_2 \coloneqq B(R_2)$.
Therefore,
it is possible to further narrow down the range of
$\bm\alpha^*$
to
$B_1\cap B_2$,
and consider a tighter upper bound
\begin{align*}
\max_{\bm\alpha\in B_1 \cap B_2\cap H} |\Xj^\top\bm\alpha| \le \min\left\{\max_{\bm\alpha\in B_1\cap H} |\Xj^\top\bm\alpha|, \max_{\bm\alpha\in B_2\cap H}|\Xj^\top\bm\alpha|\right\}.
\end{align*}
By using this tighter upper bound in the safe screening, we expect that more inactive patterns can be screened out.
For a pattern $j$, if
$\max_{\bm\alpha\in B_1 \cap B_2} |\Xj^\top\bm\alpha| <\lambda\le \min\left\{\max_{\bm\alpha\in B_1} |\Xj^\top\bm\alpha|, \max_{\bm\alpha\in B_2}|\Xj^\top\bm\alpha|\right\}$, 
then it is not possible to remove the pattern $j$ using either $B_1$ or $B_2$ alone, but it becomes possible to remove it by using the intersection of $B_1$ and $B_2$.
The following theorem indicates that
$\max_{\bm\alpha\in B_1 \cap B_2} |\Xj^\top\bm\alpha|$ 
can be expressed in a closed form, and can be computed in $\mathcal O(n)$ time.
\begin{theo}[Multiple safe screening rule]
 \label{the:multiple_safe_screening_rule}
 For any pair of primal-dual feasible solutions
 $R_1=(\tilde{\bm\beta}^{(1)}, \tilde{\beta}_0^{(1)}, \tilde{\bm\alpha}^{(1)}), R_2=(\tilde{\bm\beta}^{(2)}, \tilde{\beta}_0^{(2)}, \tilde{\bm\alpha}^{(2)})$, 
 and for any $j \in [d]$,
 it holds that 
 \[
 u_j^\prime(R_1, R_2)\coloneqq\max_{\bm\alpha\in B_1\cap B_2\cap H} |\Xj^\top\bm\alpha| = \max\{u_j^+, u_j^-\} < \lambda\Longrightarrow\beta_j^*=0,
 \]
 where 
 \[
 u_j^+\coloneqq\left\{\begin{array}{ll}
		\Xj^\top\aone + r(R_1)\|\Xj-\Pi_{\bm 1}(\Xj)\|_2, & \Xj\in C_1,\\
		       \Xj^\top\atwo + r(R_2)\|\Xj-\Pi_{\bm 1}(\Xj)\|_2, & \Xj\in C_2,\\
		       \Xj^\top\azero^\prime + r^\prime \|\Xj-\Pi_{\bm 1}(\Xj)-\Pi_{\bm \delta}(\Xj)\|_2, & \mathrm{otherwise},
 \end{array}\right.,
 \]
 \[
 u_j^-\coloneqq\left\{\begin{array}{ll}
		-\Xj^\top\aone + r(R_1)\|\Xj-\Pi_{\bm 1}(\Xj)\|_2, & -\Xj\in C_1,\\
		       -\Xj^\top\atwo + r(R_2)\|\Xj-\Pi_{\bm 1}(\Xj)\|_2, & -\Xj\in C_2,\\
		       -\Xj^\top\azero^\prime + r^\prime \|\Xj-\Pi_{\bm 1}(\Xj)-\Pi_{\bm \delta}(\Xj)\|_2, & \mathrm{otherwise},
 \end{array}\right.,
 \]
 and 
 \begin{align*}
  \bm \delta &\coloneqq \aone - \atwo,\\
  \azero^\prime &\coloneqq t\aone + (1-t)\atwo,\\
  r^\prime &\coloneqq \sqrt{r(R_2)^2 - t^2 \|\bm \delta\|_2^2},\\
  t &\coloneqq \frac{1}{2}\left(1 + \frac{r(R_2)^2 - r(R_1)^2}{\|\bm \delta\|_2^2}\right),\\
  C_1&\coloneqq \left\{\bm a\in\mathbb R^n\mid \frac{\bm a^\top\bm \delta}{\|\bm a-\Pi_{\bm 1}(\bm a)\|_2}\le\frac{r(R_2)^2 - r(R_1)^2 - \|\bm \delta\|_2^2}{2r(R_1)} \right\},\\
  C_2&\coloneqq \left\{\bm a\in\mathbb R^n\mid \frac{\bm a^\top\bm \delta}{\|\bm a-\Pi_{\bm 1}(\bm a)\|_2}\ge\frac{r(R_2)^2 - r(R_1)^2 + \|\bm \delta\|_2^2}{2r(R_2)} \right\}.
 \end{align*}
\end{theo}

The proof of this theorem is presented in Note S3.

\begin{proof}[Remark]
Instead of the safe screening rule in Theorem \ref{the:multiple_safe_screening_rule}, we may consider the safe pruning rule for two reference solutions. However, unlike the safe screening case above, we conjecture that the pruning conditions cannot be written in a closed form. So, for the safe pruning with two reference solutions, we just apply two pruning rules derived by each of the solutions:
\[
    v^\prime_j(R_1, R_2)\coloneqq\min\left\{ v_j(R_1), v_j(R_2)\right\} < \lambda\\\Longrightarrow \forall p_j\sqsubset p_k, \beta_k^* = 0.
\]

In addition, if we have three or more reference solutions, we expect that we can screen out more features.
However, we conjecture that the screening conditions become very complicated as the increase of the number of reference solutions.
\end{proof}

\subsubsection*{Multiple dynamic screening with SPP}
\label{sec:ext_dynamic}

In this section, we describe the extension of multiple safe screening to \emph{dynamic screening}~\citep{bonnefoy2015dynamic}.
Dynamic screening is a method of performing safe screening using feasible solutions obtained during optimization.
Because the performance of safe screening depends on how close the feasible solution is to the optimal solution, more patterns tend to be removed with updated solutions.
This means that, if the update to the solution is not substantial enough, the performance of safe screening may not differ significantly between before and after the updates.
In the case where multiple feasible solutions are available, all of the solutions must be sufficiently updated.
However, updating multiple solutions needs additional computational costs, so it is necessary to consider the trade-off between the cost of updating multiple solutions and the number of patterns that can be removed by safe screening.

In order to reduce the cost of updating multiple solutions, we restrict the number of multiple solution updates to $M \in \mathbb{N}$.
Specifically, we repeat updates and screening for multiple solutions up to $M$ iterations, and then keep the single solution that has the smallest duality gap, which is an indicator of proximity to the optimal solution, while the others are discarded.

\subsubsection*{SPP with multiple hyperparameter selection}
\label{sec:ext_multi_hyperparameter}

In this section, we describe a method for accelerating the computation of regularization paths for multiple hyperparameters using screening and pruning with multiple reference feasible solutions.
Specifically, we consider the regularization paths for the two hyperparameters of the Elastic Net.
When there are two hyperparameters, we can consider a two-dimensional regularization path as shown in Fig.~\ref{fig:elastic_net_regularization_path}, where the sequence of regularization parameters for the $L_1$ norm is represented by $\lambda^{(1)}, \lambda^{(2)}, \ldots$, and the sequence of regularization parameters that adjust the relative strength of the $L_2$ norm is represented by $\kappa^{(1)}, \kappa^{(2)}, \ldots$.
For a given set of hyperparameters $(\lambda^{(t)}, \kappa^{(t^\prime)})$, there are two feasible solutions that can be used as reference solutions for optimization, i.e., the optimal solution at $(\lambda^{(t-1)}, \kappa^{(t^\prime)})$ and the optimal solution at $(\lambda^{(t)}, \kappa^{(t^\prime - 1)})$.

For the hyperparamers $\lambda$ and $\kappa$, we considered the following sequences of candidates: ${\lambda^{(k)}}, {k\in[K]}$ such that $\lambda^{(1)} > \lambda^{(2)} >\cdots > \lambda^{(K)}$, and ${\kappa^{(k^\prime)}}, {k^\prime\in K^\prime}$ such that $\kappa^{(1)} < \kappa^{(2)}<\cdots < \kappa^{(K^\prime)}$, respectively. 
In many cases, as $\kappa$ increases for a fixed $\lambda^{(k)}$, the number of patterns for which $\beta^*_j \neq 0$ decreases.
Therefore, by setting ${\kappa^{(k^\prime)}}$ in this way, the number of patterns for which $\beta^*_j\neq 0$ increases as $k^\prime$ becomes larger.
Furthermore, we set $\lambda^{(1)}$ as the smallest $\lambda$ value such that $\bm\beta^* = \bm 0$ and set $\kappa^{(1)}=0$.

In the case of having two regularization parameters, iterative optimization is performed in a manner analogous to the case of a single parameter.
There are multiple possible options to optimize the regularization parameter sequence.
In this paper, we adopt an option to optimize them in the order of $(\lambda^{(1)}, \kappa^{(1)})$, $(\lambda^{(1)}, \kappa^{(2)})$, $\ldots$, $(\lambda^{(1)}, \kappa^{(K^\prime)})$, $(\lambda^{(2)}, \kappa^{(1)})$, $\ldots$, $(\lambda^{(2)}, \kappa^{(K^\prime)})$, $\ldots$, $(\lambda^{(K)}, \kappa^{(K^\prime)})$.

Let $\mathcal R$ denote the set of feasible solutions.
During optimization at $(\lambda^{(k)}, \kappa^{(k^\prime)})$, if $k>1$, the optimal solution at $(\lambda^{(k-1)}, \kappa^{(k^\prime)})$ are appended to $\mathcal R$.
On the other hand, if $k^\prime>1$, the optimal solution at $(\lambda^{(k)}, \kappa^{(k^\prime - 1)})$ are appended to $\mathcal R$.
When $|\mathcal R| = 1$, we execute safe pruning and screening using a single feasible solution in conventional way.
In contrast, when $|\mathcal R| = 2$, we execute pruning and screening using two feasible solutions.
The detailed algorithm is described in Algorithm~\ref{alg:pathwise_multi_spp} in Note S4.
Note that this approach can be easily extended to cases with three or more hyperparameters although it is not explicitly described in this paper.

\subsubsection*{SPP with Hyper-parameter selection by cross-validation}
\label{sec:ext_cross_validation}

CV is commonly used for determining hyperparameters. 
In CV, the following steps are taken to determine hyperparameters:
First, the given data set is divided into several groups.
Then, one of the groups is used as the validation set, while the remaining groups are used for model training.
Performance metrics such as prediction errors and classification accuracy are calculated for each hyperparameter(s) candidate using the validation data.
This process is repeated by sequentially swapping the validation group, and the metrics are averaged for each hyperparameter(s) candidate.
The hyperparameter candidate(s) with the highest average score is selected as the best hyperparameter(s).
In such a CV process, a sequence of optimization problems with slightly different training set are solved one by one for each hyperparameter.

Our idea here is to use optimal solutions obtained at different steps of CV as another reference feasible solutions.
Specifically, we use two reference feasible solutions obtained as the optimal solutions at different hyperparameters and at different CV steps, and perform safe screening and pruning using these two reference feasible solutions as described in \refsec{sec:ext_multi_screening}.
Let $\mathcal I^{(1)}=[n]$ denote the set of indices of the entire dataset, and consider a sequences of its subset, denoted by $\mathcal I^{(2)}, \mathcal I^{(3)}, \ldots, \mathcal I^{(K)} \subsetneq \cI^{(1)}$.
Fig.~\ref{fig:cv_regularization_path} shows a schematic diagram of this procedure, and the details are described in Algorithm~\ref{alg:pathwise_cv_spp} in Note S4.
Given a sequence of subscript sets $\{\mathcal I^{(k)}\}_{k\in[K]}$ and a sequence of hyperparameters $\{\lambda^{k^\prime}\}_{k^\prime\in[K^\prime]}$, we use the reference feasible solutions of the optimal solutions at $(\mathcal I^{(1)}, \lambda^{(k^\prime)})$ and $(\mathcal I^{(k)}, \lambda^{(k^\prime - 1)})$ during optimization of $(\mathcal I^{(k)}, \lambda^{(k^\prime)})$.
Note that, in Algorithm~\ref{alg:pathwise_multi_spp}, $\kappa$ is fixed for simplicity, but it is possible to extend it to select both $\lambda$ and $\kappa$.

\subsection*{Numerical Experiments}
\label{sec:experiments}
In this section, we describe numerical experiments that verify the effectiveness of the proposed SPP method and its extension in model selection scenario.

\subsubsection*{SPP with multiple Hyper-parameter selection}
\label{sec:exp_multi_hyperparameter}
We first conducted a comparison of computation times for two-dimensional regularization path calculations for the two regularization parameters, $\lambda$ and $\kappa$, in the Elastic Net.
We compared the performances of Single-SPP, which utilizes only a single reference feasible solution, and Multi-SPP, which uses multiple reference feasible solutions.
In the case of Multi-SPP, experiments were performed for $M \in \{0, 1, 2, 4\}$, where $M$ represents the number of times that multiple dynamic screening is executed (note that $M=0$ signifies the use of multiple solutions solely for screening and pruning at the start of optimization, with dynamic screening performed using only a single solution thereafter).
Regarding $\lambda$, we investigated cases where the number of partitions from $\lambda_{\mathrm{max}}$ to $0.01\lambda_\mathrm{max}$ was 5, 10, 20, and 40,
where $\lambda_{\mathrm{max}}$ is the smallest $\lambda$ that makes all patterns inactive (See \refsec{sec:exp-setup}).
As for $\kappa$ sequence, we investigated the cases with $\kappa \in \{0, 0.01, 0.1, 1.0, 10.0, 100.0\}$.
For both Single-SPP and Multi-SPP, the optimization was performed based on Algorithm~\ref{alg:pathwise_multi_spp} in Note S4.
In Single-SPP, only the optimal solution from the previous $\lambda$ was used as a reference feasible solution.

The experimental results are presented in Fig.~\ref{fig:result_multihp}.
From this figure, it can be confirmed that the use of multiple reference solutions is effective in many cases.
Moreover, in Multi-SPP, dynamic screening can improve performance to some extent even for small values of $M$.
The increase in $M$ did not result in a considerable increase in computational time.
Although no improvement in computation speed was observed for w1a, this may be due to the additional cost of computing multiple solutions outweighs the benefits of reducing the computation cost with multiple solutions.

\subsubsection*{SPP with Hyper-parameter selection by cross-validation}
\label{sec:exp_cv}
Next, we conducted experiments to investigate the use of reference feasible solutions in the hyperparameter selection process based on CV.
To compare the computational costs, we investigated the case where one of the two hyperparameters, $\kappa$, was fixed at $0$, and only $\lambda$ was varied as in \refsec{sec:ext_multi_hyperparameter}.
In terms of CV configuration, we compared the computation time of Leave-One-Out CV (LOOCV).
Specifically, we constructed 10 leave-one-out datasets at random, and compared the relative computational costs of each method option and problem setting.
For a Single-SPP with one reference feasible solution, we used the optimal solution from the previous $\lambda$ as in \refsec{sec:exp_multi_hyperparameter}.

Fig.~\ref{fig:result_cv} shows the experimental results.
In Multi-SPP, a significant reduction in overall computation time can be achieved by using the optimal solution with the entire dataset as a reference feasible solution.
While there were some cases where dynamic screening in Multi-SPP showed some effectiveness, no significant changes in performance were observed in many other cases. 
We conjecture that this is due to a trade-off between the reduction in computation time resulting from the effectiveness of screening with increasing $M$ and the increase in computation time necessary for updating multiple optimal solutions.

\subsubsection*{Comparison with boosting-based methods}
\label{sec:exp_gboost}
Finally, we compared the computational costs of the proposed SPP and the existing boosting-based approach in predictive pattern mining~\citep{saigo2007mining,saigo2009linear,saigo2009gboost}. 
The boosting-based approach involves sequential addition of patterns to the prediction model, necessitating tree traverse search at each step.
In contrast, the SPP requires only a single tree traverse search (under fixed regularization parameter), thereby exhibiting a computational advantage.

In order to conduct fair comparisons, we set the problem such that both methods start from $\lambda_{\rm max}$ and seek the optimal solution at $0.01\lambda_{\rm max}$.
For the boosting-based method, we measured the computational cost taken from the optimal solution at $\lambda_{\rm max}$, which does not include any patterns, to adding one pattern at a time until arriving at the optimal solution at $0.01\lambda_{\rm max}$.
For the SPP, we considered the one-dimensional regularization path from $\lambda_{\rm max}$ to $0.01\lambda_{\rm max}$ and measured the computational cost when performing the same process as in \refsec{sec:exp_multi_hyperparameter}.
Both methods used the coordinate gradient descent method~\citep{tseng2009coordinate} for optimization, and the another hyperparameter $\kappa$ was set to $0$.

First, we considered the problems of graph classification and graph regression with chemical compound datasets as examples in predictive graph mining. 
Specifically, we used two datasets for graph classification:
CPDB (``Helma CPDB Mutagenicity Subset'', $n=684$) and
Mutagenicity (``Bursi Mutagenicity Dataset'', $n=4337$),
and two datasets for graph regression:
Bergstrom (``Bergstrom Melting Point Dataset'', $n=185$) and
Karthikeyan (``Karthikeyan Melting Point Dataset'', $n=4450$).
All datasets were retrieved from \url{http://cheminformatics.org/datasets/}.
Note that, since these datasets were downloaded
when our preliminary work \citep{nakagawa2016safe} was conducted,
and the website above were closed later,
we also present the link to the archived website:
\url{http://web.archive.org/web/20150503130239/http://cheminformatics.org/datasets/}.
In addition, since the number of instances $n$ were mistakenly noted in the preliminary work,
we noted collect numbers.
%
%Specifically, we used two datasets for graph classification: \textsf{CPDB} (n=648) and \textsf{Mutagenicity} (n=4377), and two datasets for graph regression: \textsf{Bergstrom} (n=185) and \textsf{Karthikeyan} (n=4173).
%
Fig.~\ref{fig:time_graph_mining} shows the computational cost of the boosting-based method (\textsf{boosting}) and the SPP (\textsf{SPP}).
In all the cases, \textsf{SPP} is faster than \textsf{boosting}, and the differences become more significant as the maximum length of patterns increases.
The results also indicates that the tree traverse time are not so different between the two methods. 
We conjecture that this is because the most time-consuming part of gSpan is the isomorphism check which is required to avoid enumerating duplicated graphs.

Next, we considered classification and regression in item-set predictive mining.
We used two datasets for classification: \textsf{splice}($n=1000$) and \textsf{a9a}($n=32561$), and two datasets for regression: \textsf{dna}($n=2000$) and \textsf{protein}($n=6621$).
The results are shown in Fig.~\ref{fig:time_itemset_mining}.
In all the cases, \textsf{SPP} was faster than \textsf{boosting}.
Unlike graph mining cases, the tree traverse time of the SPP was smaller than that of boosting-based method.
We conjecture that this is because boosting-based methods require multiple enumeration of patterns, while the SPP requires only one enumeration.

\section*{Discussion}

%As stated in \refsecs{sec:introduction}{sec:proposed_method}, conventional safe screening can check whether a pattern can be active or not, it is not realistic to apply it if the number of patterns is huge.
%With the monotonicity of the patterns (Lemma \ref{lem:monotnicity}) we can safely screen out {\em a set of} patterns with a small increase of computational cost.
%Before SPP is proposed, there has been a boosting-based method that sequentially adds patterns needed for good predictions (\refsec{sec:exp_gboost}). However, we conjectured that it is costly compared to SPP in the sense that a tree traversal is needed for each addition of a pattern, and the experimental results in \refsec{sec:exp_gboost} proved the efficiency of SPP.
%
%Also, we picked up several model selection tasks in the predictive pattern mining and their efficient implementations in \refsec{sec:sec4}. Especially, we proved that SPP works more efficiently if there are two reference feasible solutions, which is also proved experimentally in \refsecs{sec:exp_multi_hyperparameter}{sec:exp_cv}.
%In \refsec{sec:ext_dynamic}, we have also discussed that, the parameter $M$ that defines the number of applying the dynamic safe screening with multiple solutions should be properly chosen. In the experiments of \refsecs{sec:exp_multi_hyperparameter}{sec:exp_cv}, we found that not only small but also large $M$ may increase the computational cost. To determine a proper $M$ is a future work.

Structured data such as sets, graphs, and sequences are common in many fields, and it is necessary to develop machine learning methods to handle such structure data. Although neural networks for structured data have seen significant recent development and can achieve good predictive performances, practical problems often require interpretation and explanation of the model behavior and extracting important features. In this study, we propose a pattern mining approach for machine learning modeling of structured data that can both prediction ability and explainability.

The challenge of extracting features from structural data lies in the exponential increase in the number of potential substructures that can serve as features. This difficulty has led to the development of various pattern mining algorithms, particularly in the task of enumeration. However, research that integrates pattern mining with predictive modelings, such as regression or classification, is limited. To our knowledge, only boosting-based approaches have been proposed, which are inefficient due to the requirement of a tree traverse for each additional feature.

In this study, we addressed a common problem in the field of pattern mining by introducing safe screening. Safe screening is a technique developed in the field of sparse modeling that allows for the identification of redundant features before solving the optimization problem, which can significantly reduce the computational cost. However, as noted in \refsecs{sec:introduction}{sec:proposed_method}, applying conventional safe screening to many patterns is not feasible. To address this, we proposed the SPP method, which efficiently handles multiple features in a single tree traverse, as opposed to a boosting-based approach. By leveraging the monotonicity of patterns, as described in Lemma \ref{lem:monotnicity}, the proposed SPP method can safely eliminate a subset of patterns with only a moderate increase in computational cost.

Furthermore, in this study, we demonstrated the effectiveness of the proposed SPP method in the entire model building process, including hyperparameter selection and cross-validation. Especially, it has been observed that the SPP method works more efficiently if two reference feasible solutions are available, as experimentally demonstrated in \refsecs{sec:exp_multi_hyperparameter}{sec:exp_cv}. Additionally, in \refsec{sec:ext_dynamic}, we discussed that the parameter $M$, which determines the number of dynamic safe screening applications with multiple solutions, needs to be appropriately chosen. The experiments in \refsecs{sec:exp_multi_hyperparameter}{sec:exp_cv} have shown that not only small but also large values of $M$ may increase computational costs. Therefore, determining the appropriate value of $M$ is a topic for future work.

\section*{Experimental Procedures}

\subsection*{Resource availability}

The code for reproducing all the results is available at:
\begin{itemize}
\item \url{https://github.com/takeuchi-lab/pmopt} (for \refsecs{sec:exp_multi_hyperparameter}{sec:exp_cv})
\item \url{https://github.com/takeuchi-lab/SafePatternPruning} (for \refsec{sec:exp_gboost}; same as our preliminary work \citep{nakagawa2016safe})
\end{itemize}

\subsection*{Experimental setup} \label{sec:exp-setup}
First, we describe the settings common to all the experiments.
We compared the computation time of the entire or partial regularization path with respect to the hyperparameters $\lambda$ and/or $\kappa$.
Regarding $\lambda$, we defined $\lambda_\mathrm{max}$ as the largest value of $\lambda$ for which $\bm\beta^*=\bm 0$, and constructed a sequence of $\lambda$s by partitioning the interval from $\lambda_\mathrm{max}$ to $0.01\lambda_\mathrm{max}$ into equally spaced values on a logarithmic scale, where the number of partitions is varied depending on the experimental options.
The datasets used in the experiments in \refsecs{sec:exp_multi_hyperparameter}{sec:exp_cv} are presented in Table~\ref{tab:dataset}.
The coordinate descent method was employed for optimization, with a convergence criterion $\epsilon = 10^{-4}$.
During optimization, dynamic screening was performed every other iteration for the first $T=5$ cycles, and subsequently executed once every ten iterations.
PrefixSpan was employed as the mining algorithm for both set and sequence mining tasks.

\clearpage
\section*{Supplemental information}

\begin{description}
\item Note S1. Proof of Lemma \ref{lem:range_dual_optimal}
\item Note S2. Proof of Lemma \ref{lem:safe_screening_rule}
\item Note S3. Proof of Theorem \ref{the:multiple_safe_screening_rule}
\item Note S4. Algorithms
\end{description}

\section*{Acknowledgments}

This work was partially supported by MEXT KAKENHI (20H00601), JST CREST (JPMJCR21D3, JPMJCR22N2), JST Moonshot R\&D (JPMJMS2033-05), JST AIP Acceleration Research (JPMJCR21U2), NEDO (JPNP18002, JPNP20006) and RIKEN Center for Advanced Intelligence Project.

%\section*{Author contributions}
%
%\begin{itemize}
%\item Conceptualization: T.Y., K.N., K.Ts., I.T.
%\item Methodology: T.Y., H.H., K.N., K.Ta., I.T.
%\item Software: T.Y., K.N.
%\item Validation: T.Y., H.H.
%\item Formal Analysis: T.Y., H.H., I.T.
%\item Writing – Original Draft: T.Y., H.H., K.N., I.T.
%\item Writing – Review \& Editing: H.H., K.Ta., K.Ts., I.T.
%\item Visualization: T.Y., K.N.
%\item Supervision: H.H., I.T.
%\item Project Administration: H.H., I.T.
%\item Funding Acquisition: I.T.
%\end{itemize}
%
%\section*{Declaration of interests}
%
%The authors declare no competing interests.

\bibliography{reference}

\clearpage
\section*{Figure captions}

\begin{itemize}
\item Figure \ref{fig:pattern_examples}. Examples of patterns (sub-structures) of (a) set data, (b) graph data, and (c) sequence data.
\item Figure \ref{fig:tree_pruning}. Conceptual diagram of pruning in the search for patterns of (a) sets, (b) graphs, (c) sequences represented by a tree.
\item Figure \ref{fig:elastic_net_regularization_path}. Schematic illustration of (a) Regularization path for the $L_1$-norm regularization parameter $\lambda$. and (b) Two-dimensional regularization path for the $L1$-norm regularization parameter $\lambda$ and the relative regularization parameter $\kappa$ for the $L2$-norm in the Elastic Net.
 The rectangles in each cell represent the $\beta^*_j$ at the corresponding regularization parameter.
 The color of each rectangle indicates the value of $\beta^*_j$ where red/blue shows its signs while the thickness of the color indicates the absolute value (llustrating the increase of active (non-zero) coefficients and their absolute values as $\lambda$ decreases).
 Note that, in (b), when $\lambda = \lambda^{(1)}$ and $\bm\beta^*=\bm 0$, it is not necessary to change $\kappa$.
\item Figure \ref{fig:cv_regularization_path}. A schematic illustration of how to use feasible solutions in hyperparameter selection based on CV.
 The left rectangle represents the training data, with white corresponding to the training data and red to the validation data.
 In using multiple solutions for CV setting, we not only use the optimal solution at the previous regularization parameter, but also use the optimal solution trained with the entire data.
\item Figure \ref{fig:result_multihp}. Computation time for the entire regularization path for each dataset.
 The horizontal axis shows how many partitions of $\lambda$ were made.
 It can be confirmed that the use of multiple solutions is effective in most cases.
 In addition, the Multi-dynamic screening also often leads to a reduction in computation time.
\item Figure \ref{fig:result_cv}. The computation time for leave-one-out cross-validation for each dataset.
 The use of multiple solutions is effective for all the cases. 
 The Multi-dynamic screening is effective in settings where the number of $\lambda$ is large.
\item Figure \ref{fig:time_graph_mining}. Computational time comparison for graph classification and regression.
    The horizontal axis represents the maximum length of patterns that are mined. Each bar contains computational time taken in the tree traverse (traverse) and the optimization procedure (solve) respectively.
\item Figure \ref{fig:time_itemset_mining}. Computational time comparison for item-set classification and regression.
    The horizontal axis represents the maximum length of patterns that are mined. Each bar contains computational time taken in the tree traverse (traverse) and the optimization procedure (solve) respectively.
\end{itemize}

\section*{Table caption}

\begin{itemize}
\item Table \ref{tab:dataset}. The list of dataset used in the experiments in \refsec{sec:experiments}.
\end{itemize}

\clearpage

% ---------------------------------------------------------------------------
% fig1
% ---------------------------------------------------------------------------	
\clearpage

\begin{figure}[p]
 \begin{center}
  \subfloat[][set]{\includegraphics[width=0.27\linewidth]{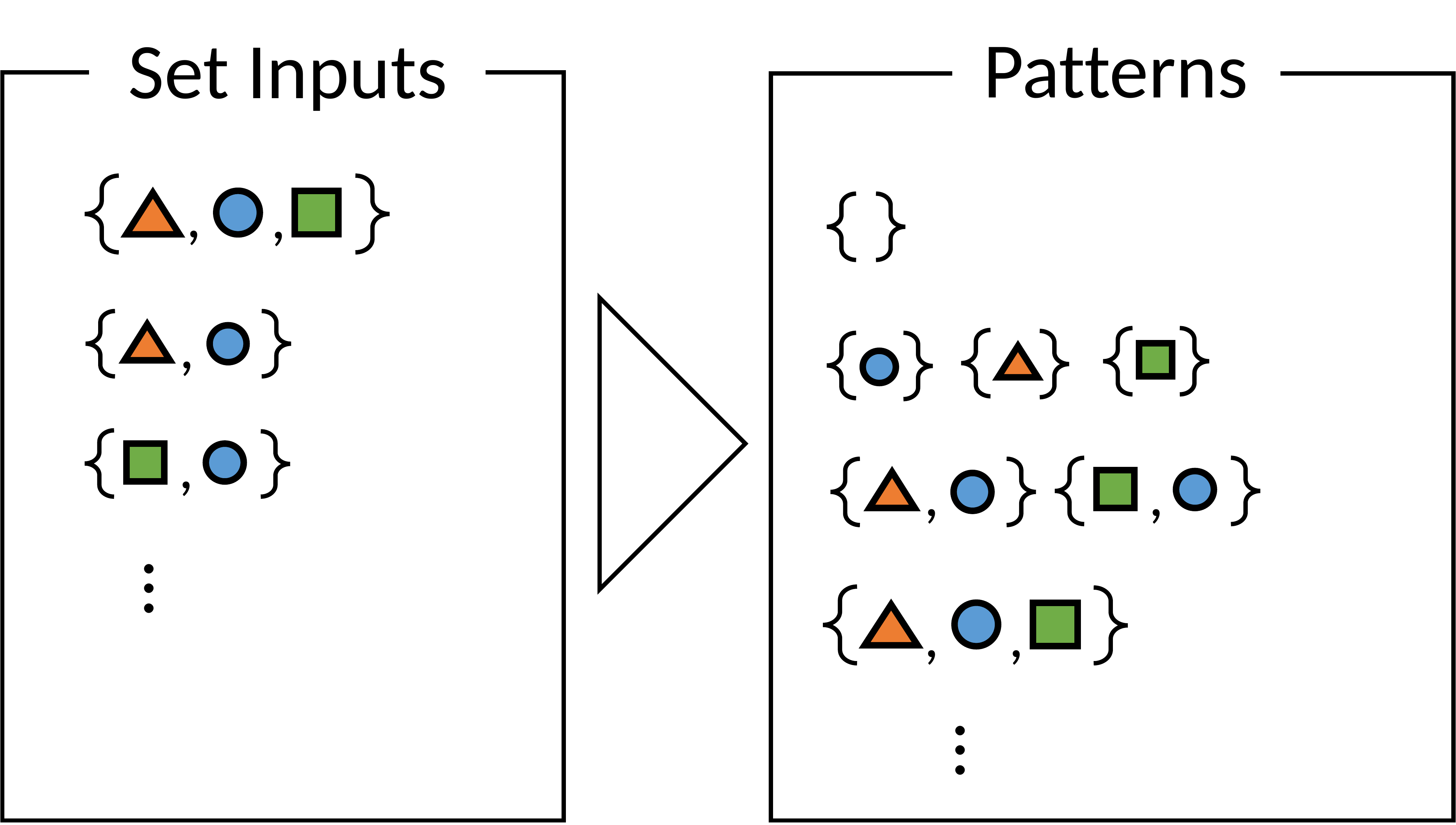}}
  \quad
  \subfloat[][graph]{\includegraphics[width=0.23\linewidth]{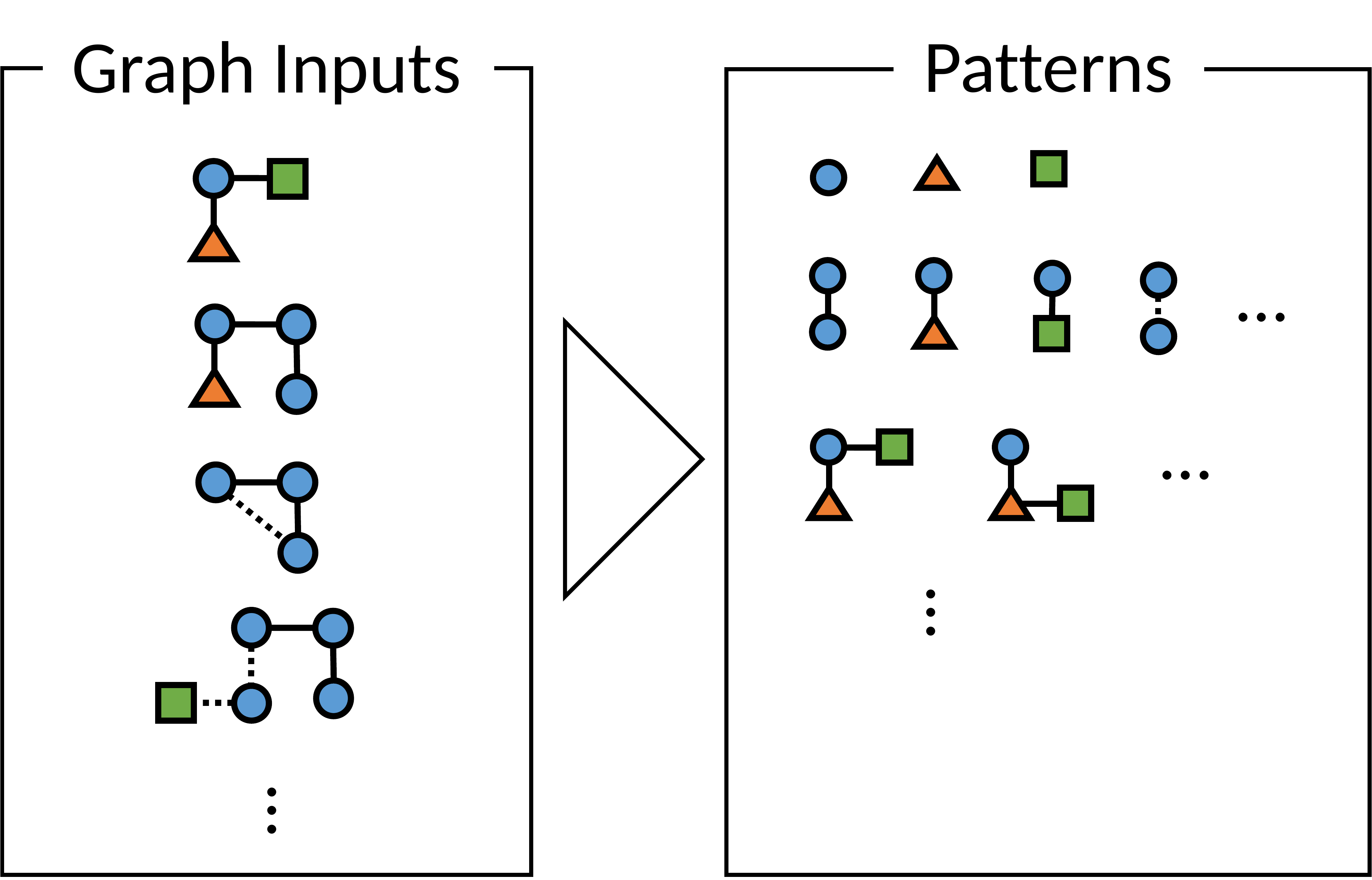}}
  \quad
  \subfloat[][sequence]{\includegraphics[width=0.41\linewidth]{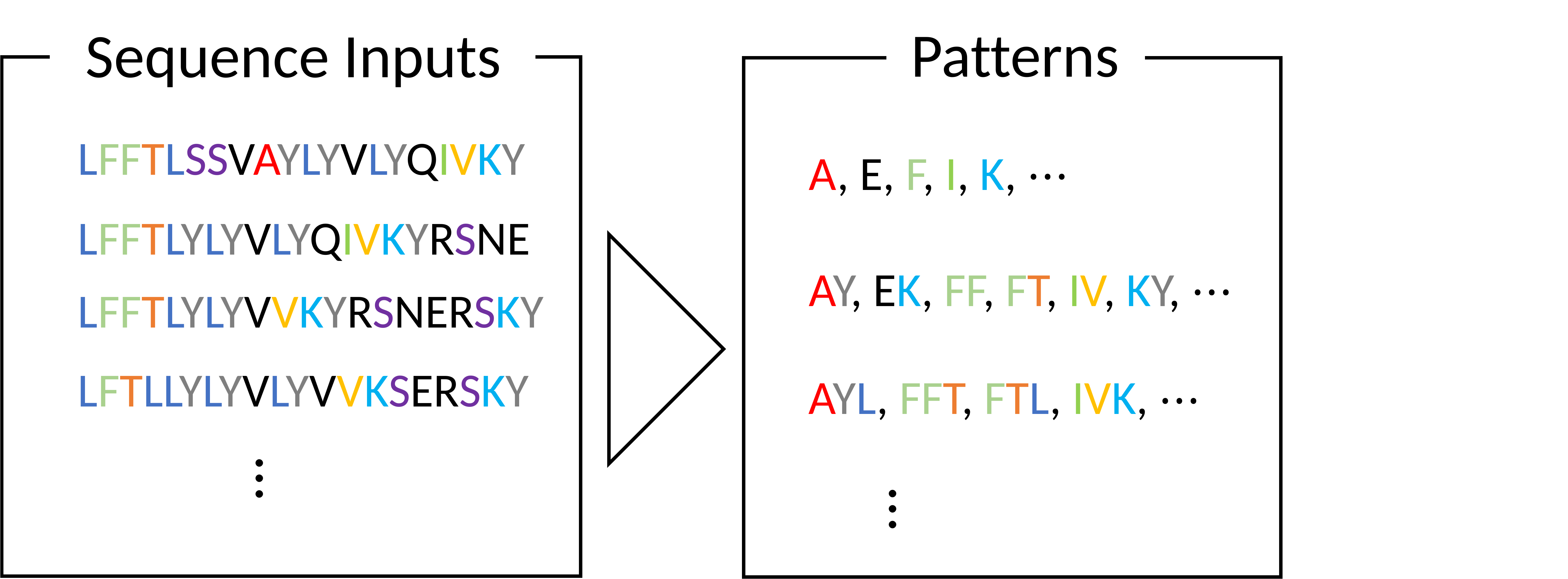}}
  \caption{Examples of patterns (sub-structures) of (a) set data, (b) graph data, and (c) sequence data.}
  \label{fig:pattern_examples}
 \end{center}
\end{figure}

\clearpage

\begin{figure}[p]
 \begin{center}
    \subfloat[][set]{\includegraphics[width=0.33\linewidth]{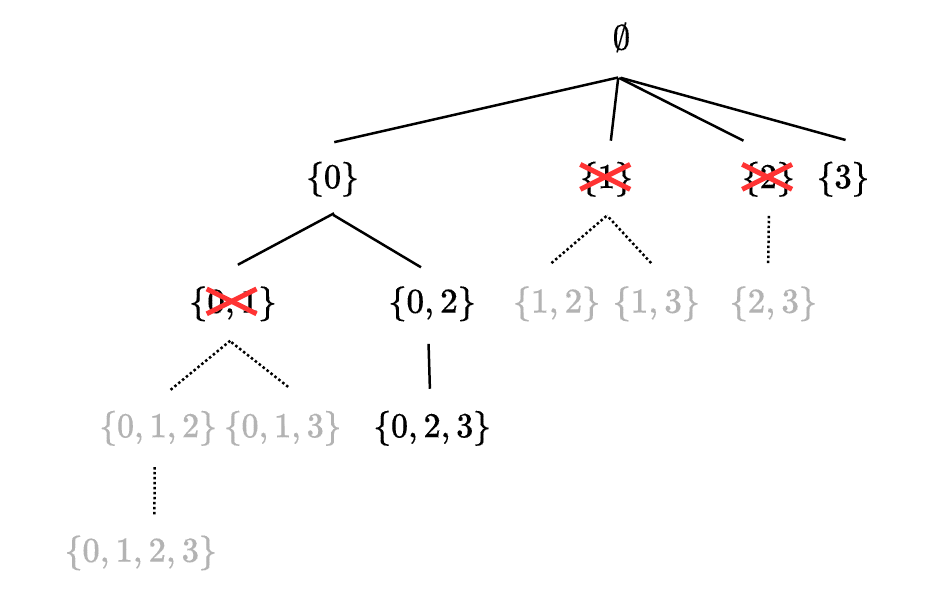}}
    \subfloat[][graph]{\includegraphics[width=0.33\linewidth]{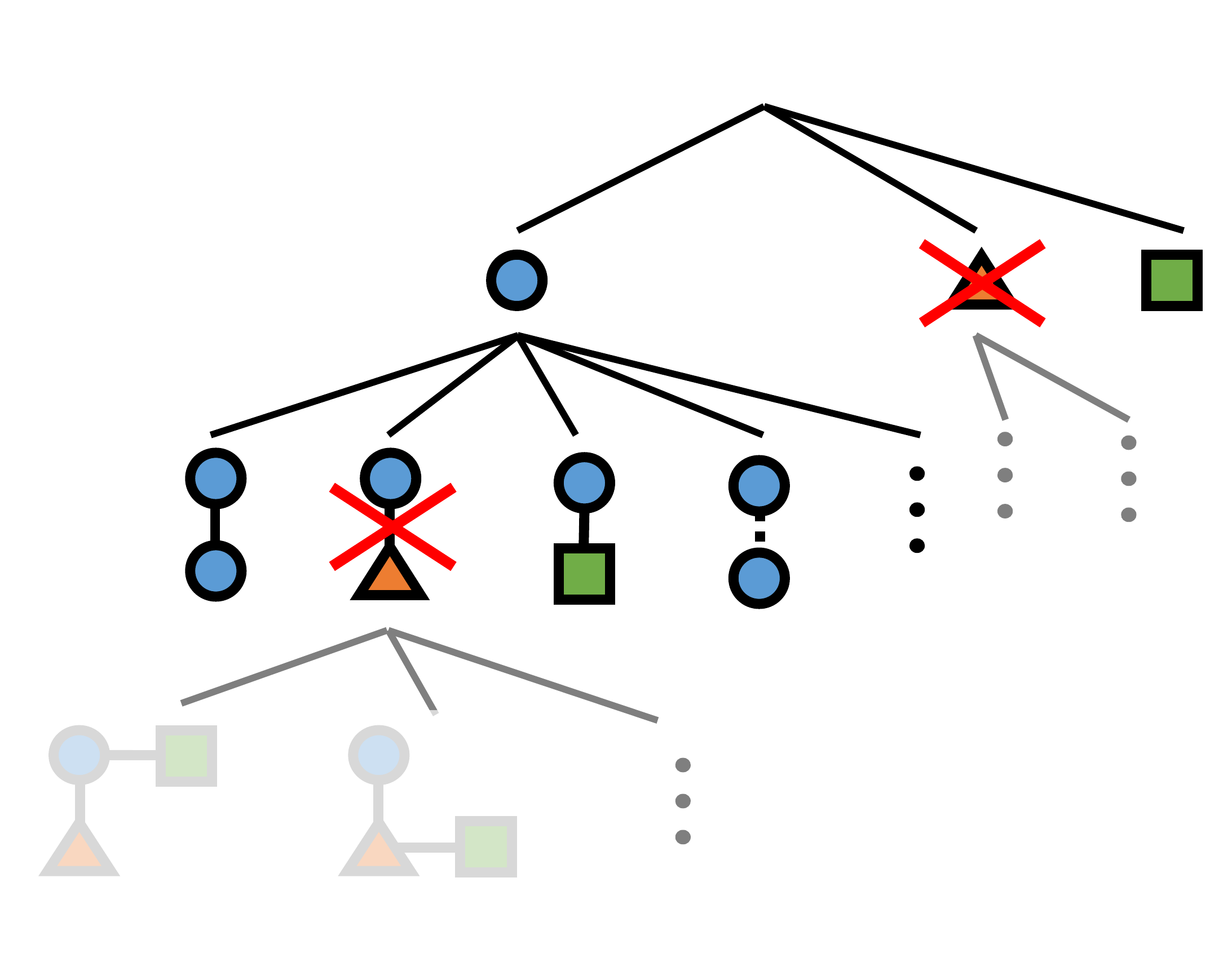}}
    \subfloat[][sequence]{\includegraphics[width=0.33\linewidth]{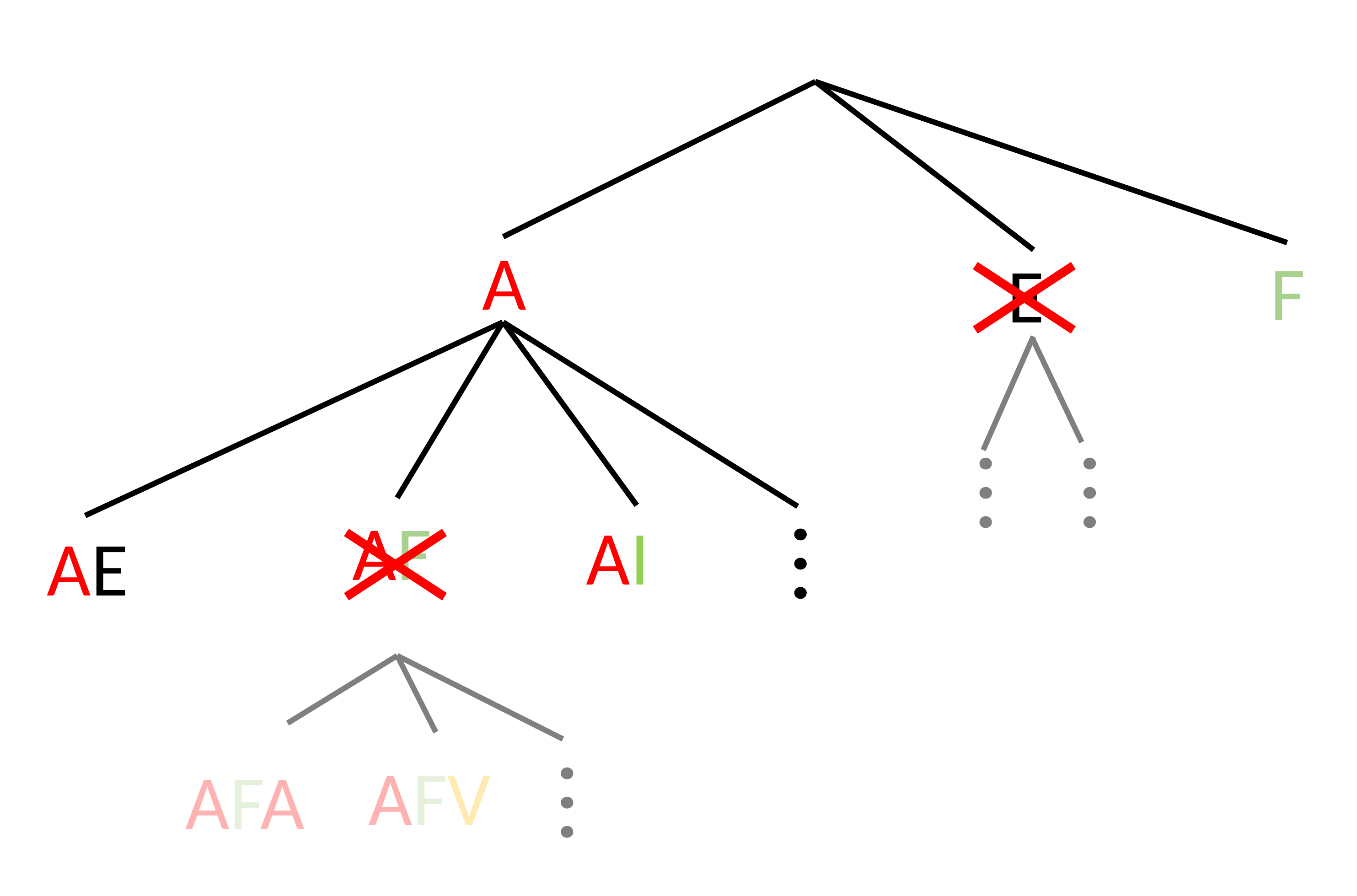}}
  \caption{Conceptual diagram of pruning in the search for patterns of (a) sets, (b) graphs, (c) sequences represented by a tree.}
  \label{fig:tree_pruning}
 \end{center}
\end{figure}

\clearpage

\begin{figure}[p]
 \centering
 \subfloat[][1-dimensional regularization path\label{fig:1d_regularization_path}]{\includegraphics[width=0.5\hsize]{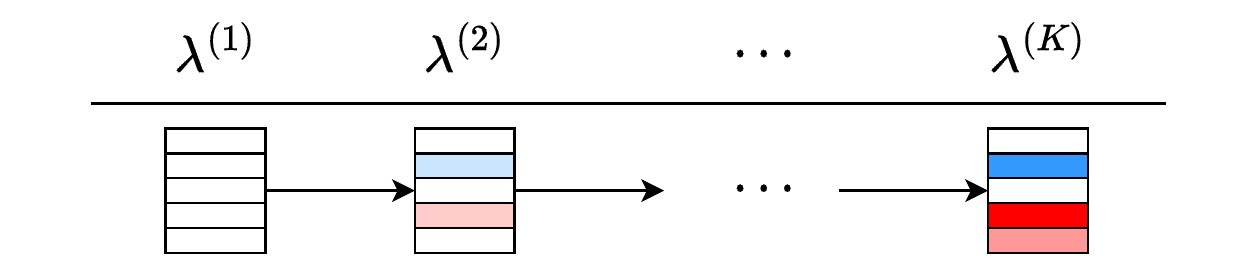}}
 \subfloat[][2-dimensional regularization path\label{fig:2d_regularization_path}]{\includegraphics[width=0.5\hsize]{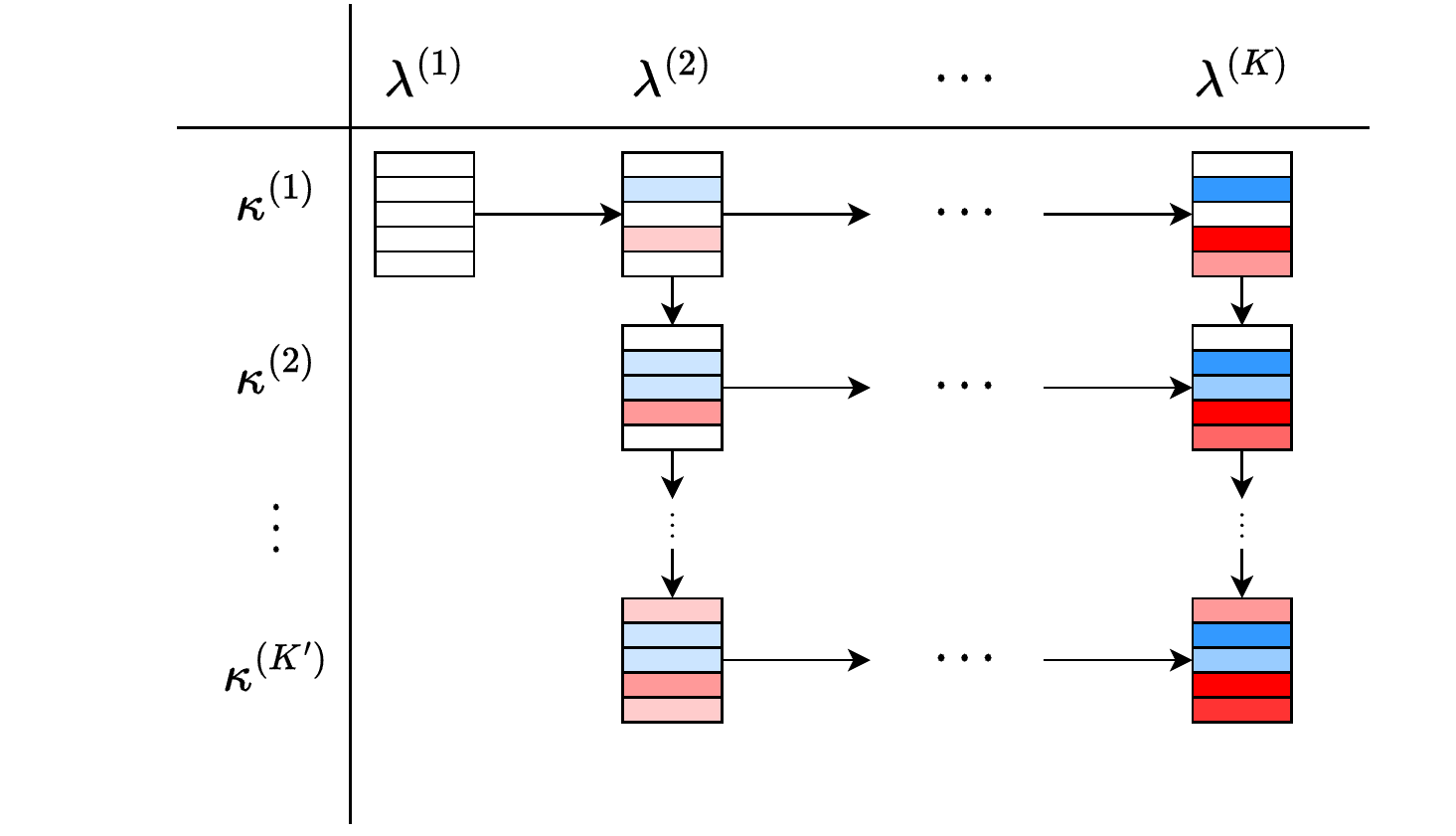}}
 \caption{Schematic illustration of (a) Regularization path for the $L_1$-norm regularization parameter $\lambda$. and (b) Two-dimensional regularization path for the $L1$-norm regularization parameter $\lambda$ and the relative regularization parameter $\kappa$ for the $L2$-norm in the Elastic Net.
 The rectangles in each cell represent the $\beta^*_j$ at the corresponding regularization parameter.
 The color of each rectangle indicates the value of $\beta^*_j$ where red/blue shows its signs while the thickness of the color indicates the absolute value (llustrating the increase of active (non-zero) coefficients and their absolute values as $\lambda$ decreases).
 Note that, in (b), when $\lambda = \lambda^{(1)}$ and $\bm\beta^*=\bm 0$, it is not necessary to change $\kappa$.
 }
 \label{fig:elastic_net_regularization_path}
\end{figure}

\clearpage

\begin{figure}[p]
 \centering
 \includegraphics[width=\hsize]{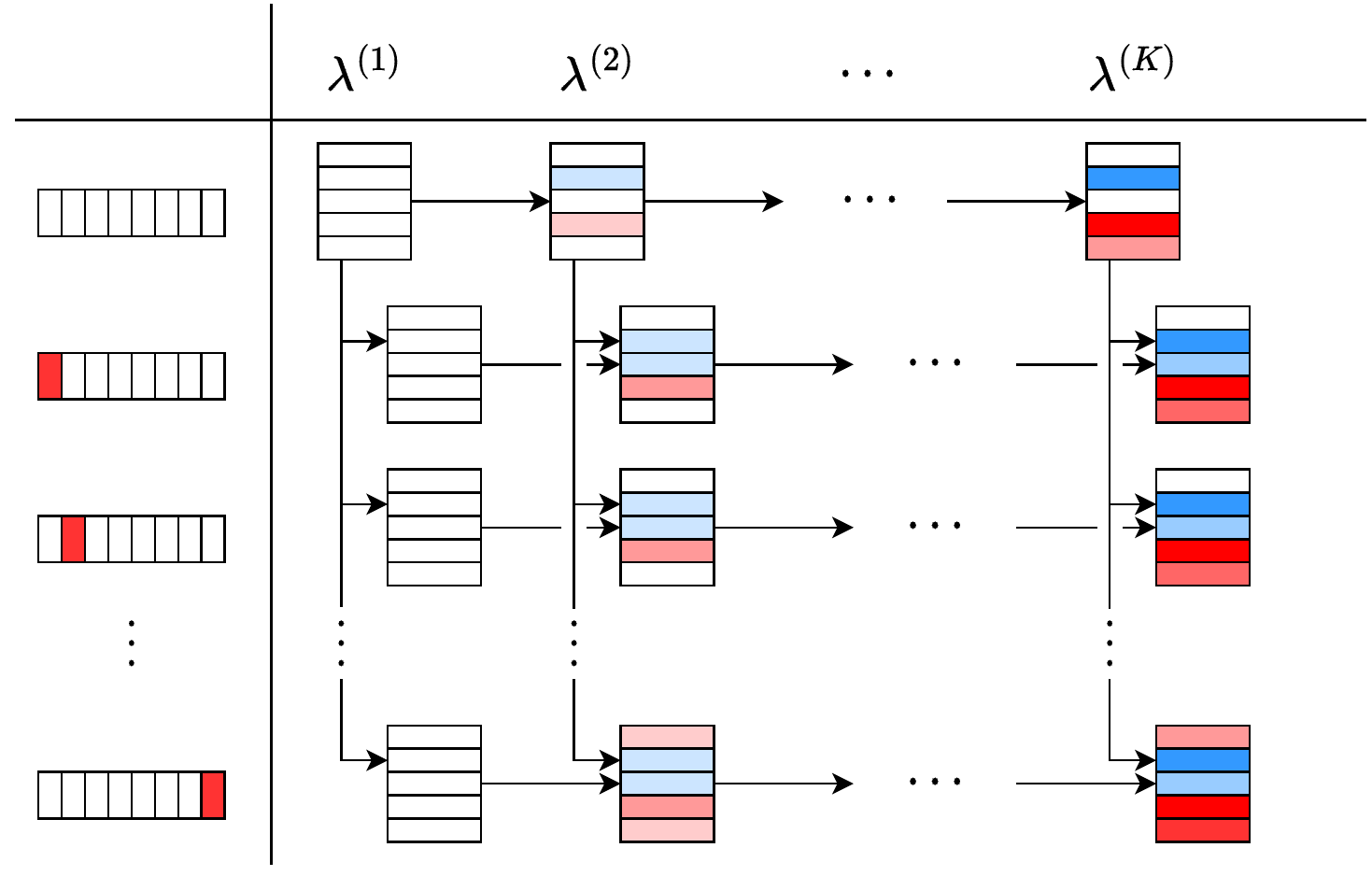}    
 \caption{A schematic illustration of how to use feasible solutions in hyperparameter selection based on CV.
 The left rectangle represents the training data, with white corresponding to the training data and red to the validation data.
 In using multiple solutions for CV setting, we not only use the optimal solution at the previous regularization parameter, but also use the optimal solution trained with the entire data.
 }
    \label{fig:cv_regularization_path}
\end{figure}

\clearpage

% \begin{figure}[p]
%     \centering
%     \includegraphics[width=\hsize]{Fig/time_multiple.pdf}
%     \caption{Computational time for each dataset. \textsf{Single Feasible} uses a single reference for screening and pruning based on the lemma \ref{lem:safe_screening_rule} and the theorem \ref{the:safe_pattern_pruning_rule}. \textsf{Multiple Feasible ($M=0$)} and \textsf{Multiple Feasible ($M=2$)} use multiple references. The horizontal axis represents how finely divided between $\lambda_\mathrm{max}$ and $0.01\lambda_{\mathrm{max}}$, and the vertical axis represents how much time reduced by multiple references from the single one.}
%     \label{fig:result_time_multple}
% \end{figure}

% \newpage

% \begin{figure}[p]
%     \centering
%     \includegraphics[width=\hsize]{Fig/patterns_multiple.pdf}
%     \caption{The number of patterns after multiple pruning. \textsf{Single Feasible} uses a single reference for screening and pruning based on the lemma \ref{lem:safe_screening_rule} and the theorem \ref{the:safe_pattern_pruning_rule}. \textsf{Multiple Feasible ($M=0$)} uses multiple references. The horizontal axis represents how finely divided between $\lambda_\mathrm{max}$ and $0.01\lambda_{\mathrm{max}}$, and the vertical axis represents the number of patterns which are enumerated by multiple pruning per regularization parameter.}
%     \label{fig:result_pattern_multple}
% \end{figure}

% \newpage

\begin{table}[p]
    \centering
    \caption{The list of dataset used in the experiments in \refsec{sec:experiments}.}
    \label{tab:dataset}
    \begin{tabular}{ccccc}\hline
        Dataset & $n$ & Structure-type & Maximum length of patterns & Problem \\\hline
        a1a & 1605 & Item-set & 5 & Classification \\
        a9a & 32561 & Item-set & 5 & Classification \\
        dna & 2000 & Item-set & 3 & Regression \\
        splice & 1000 & Item-set & 3 & Classification \\
        w1a & 2477 & Item-set & 3 & Classification \\
        rhodopsin & 1162 & Sequence & 50 & Regression \\\hline
    \end{tabular}
\end{table}

\clearpage

\begin{figure}[p]
 \centering
 \subfloat[][a1a]{\includegraphics[width=0.33\hsize]{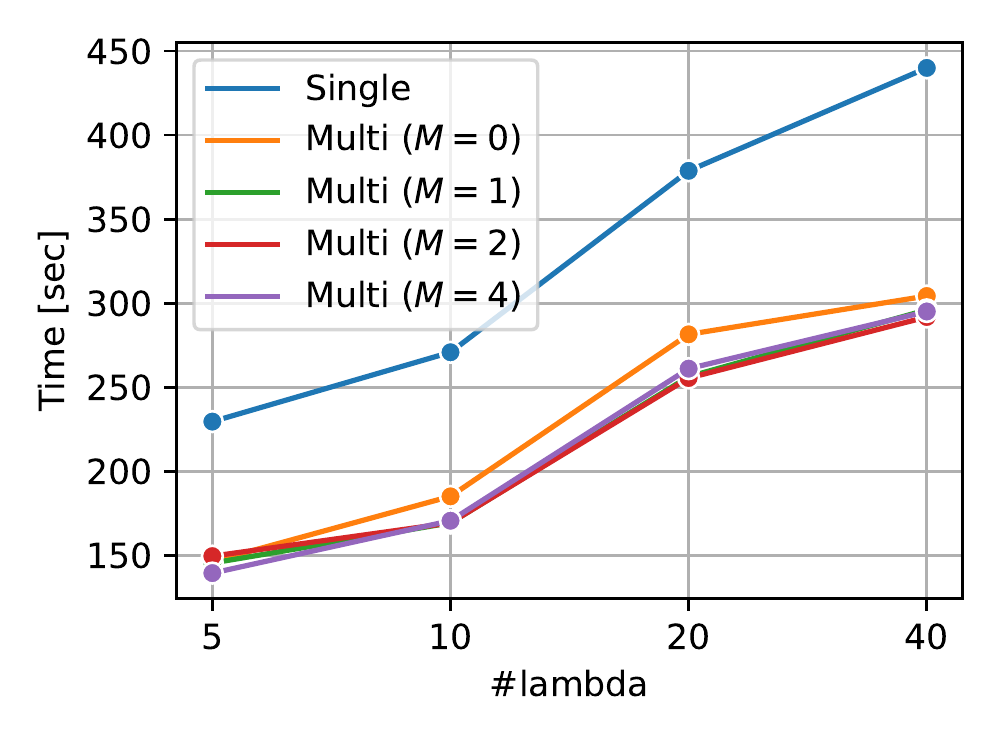}}
 \subfloat[][a9a]{\includegraphics[width=0.33\hsize]{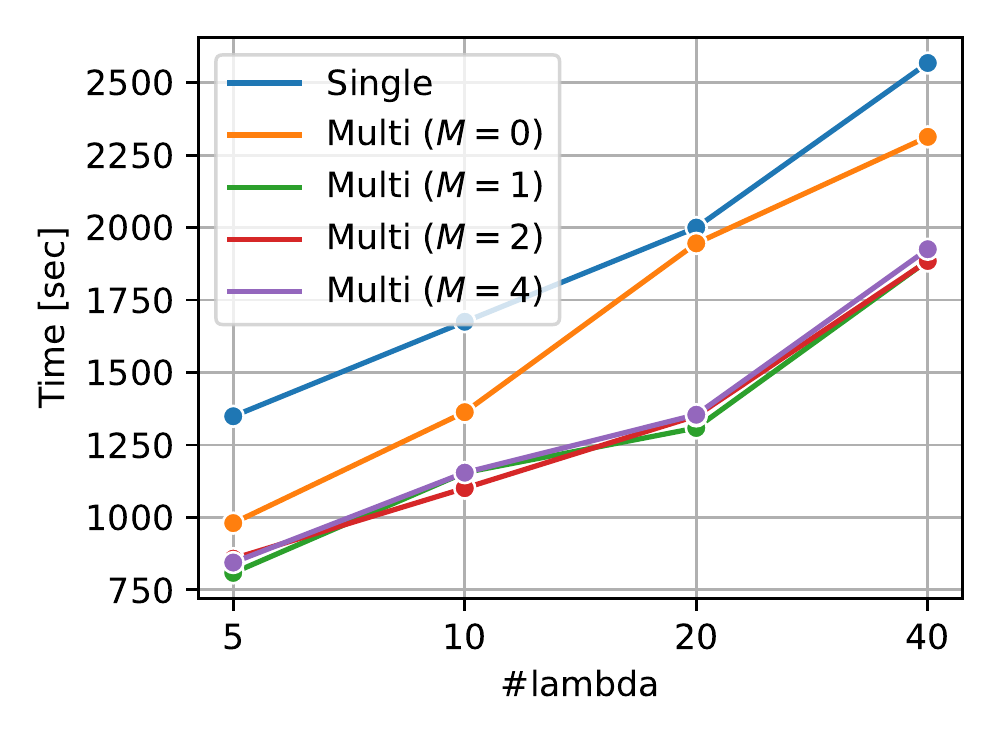}}
 \subfloat[][dna]{\includegraphics[width=0.33\hsize]{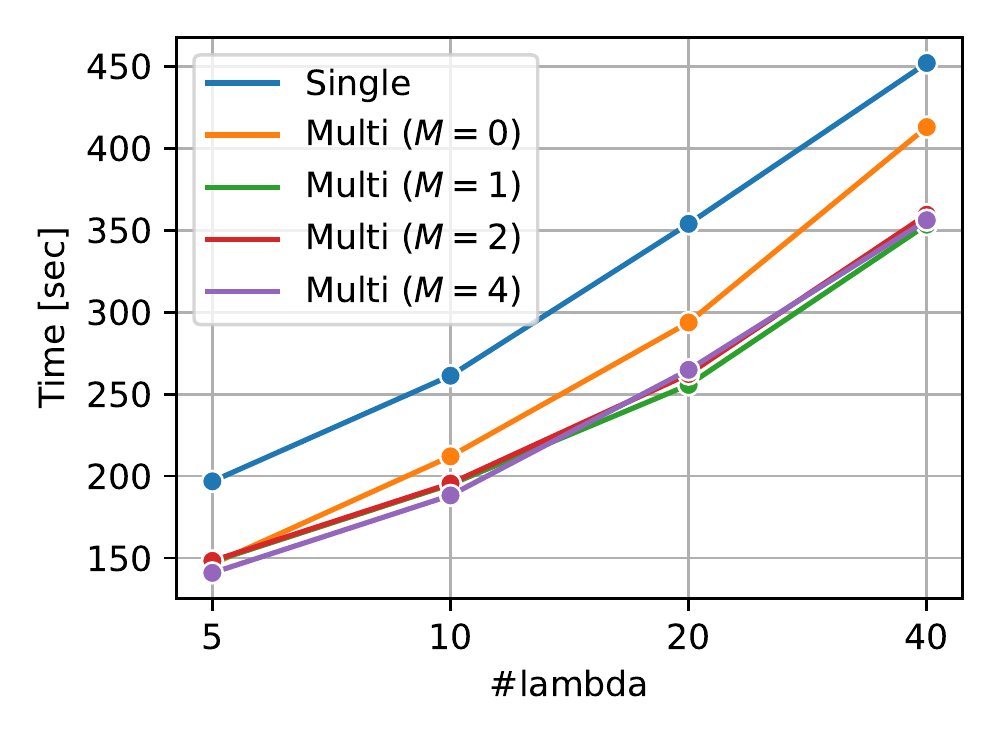}} \\
 \subfloat[][splice]{\includegraphics[width=0.33\hsize]{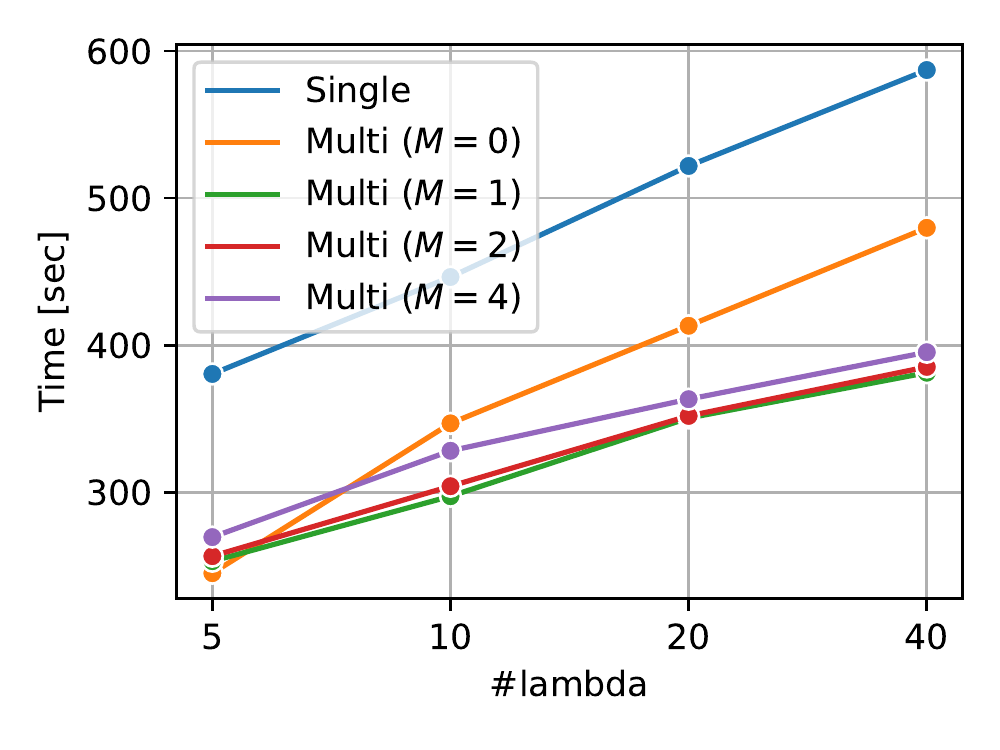}}
 \subfloat[][w1a]{\includegraphics[width=0.33\hsize]{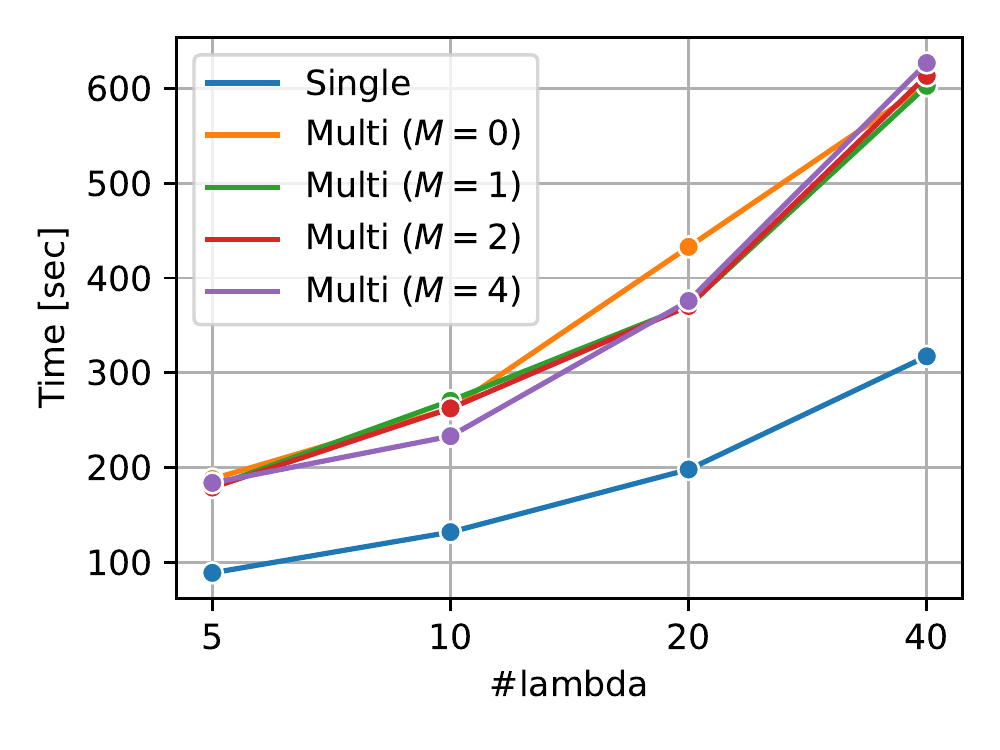}}
 \subfloat[][rhodopsin]{\includegraphics[width=0.33\hsize]{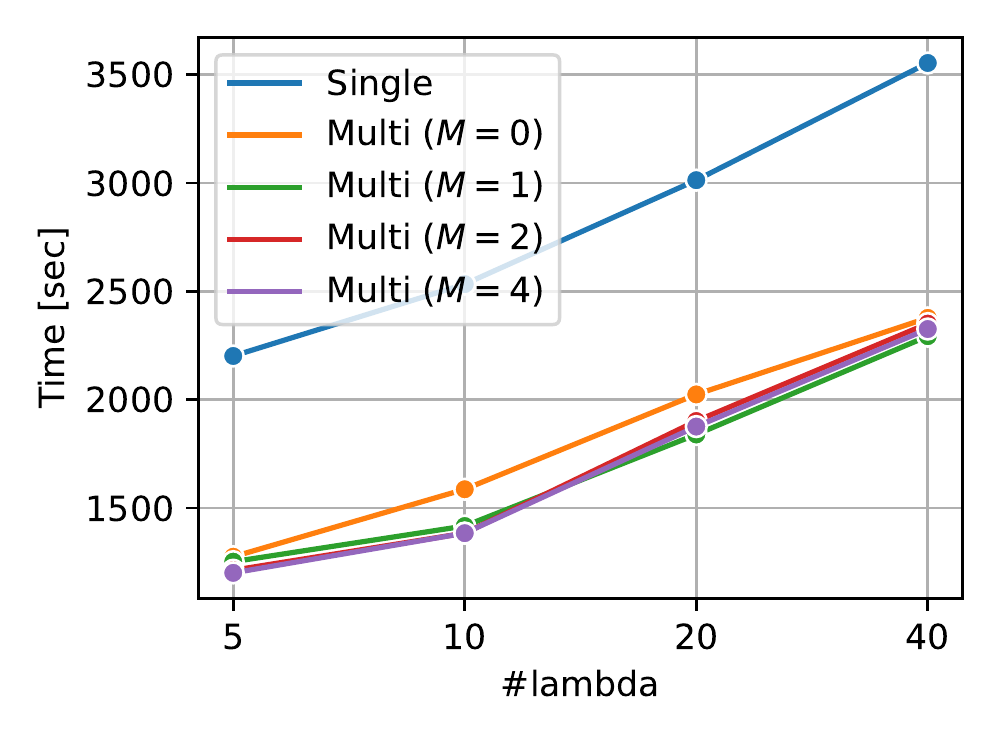}}
 \caption{Computation time for the entire regularization path for each dataset.
 The horizontal axis shows how many partitions of $\lambda$ were made.
 It can be confirmed that the use of multiple solutions is effective in most cases.
 In addition, the Multi-dynamic screening also often leads to a reduction in computation time.
 }
    \label{fig:result_multihp}
\end{figure}

\clearpage

\begin{figure}[p]
 \centering
 \subfloat[][a1a]{\includegraphics[width=0.33\hsize]{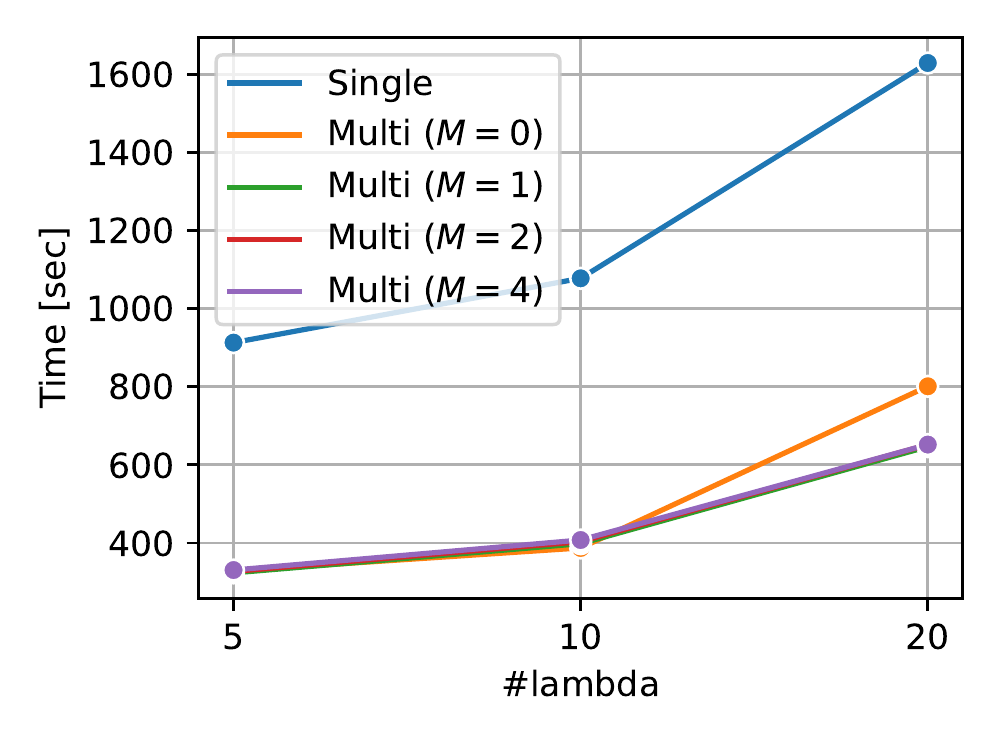}}
 \subfloat[][a9a]{\includegraphics[width=0.33\hsize]{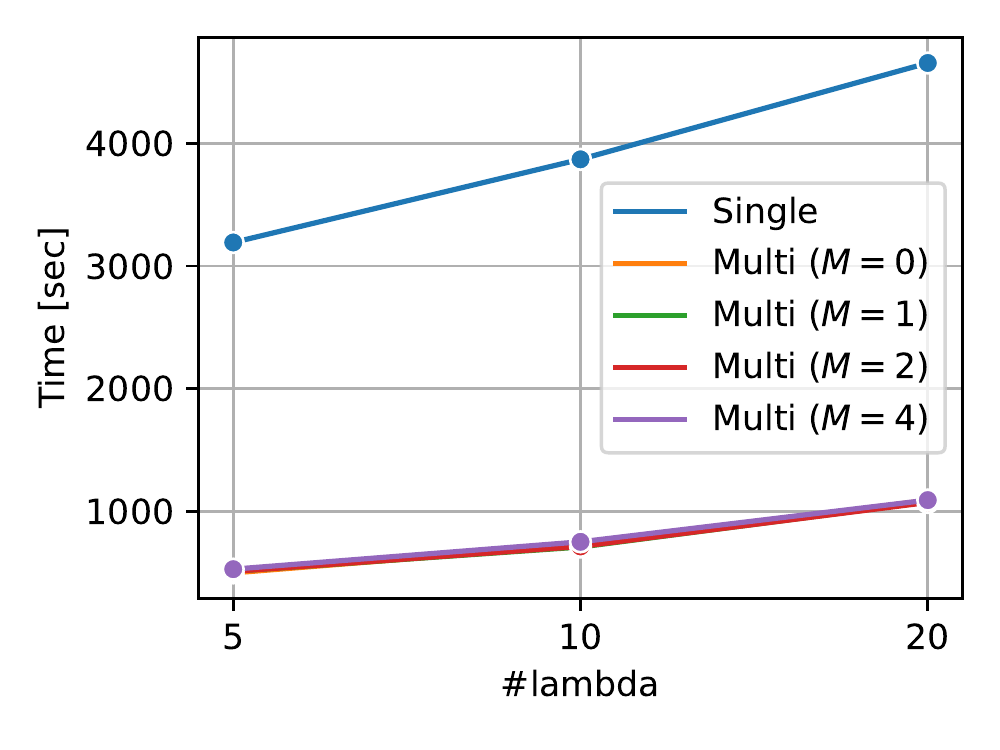}}
 \subfloat[][dna]{\includegraphics[width=0.33\hsize]{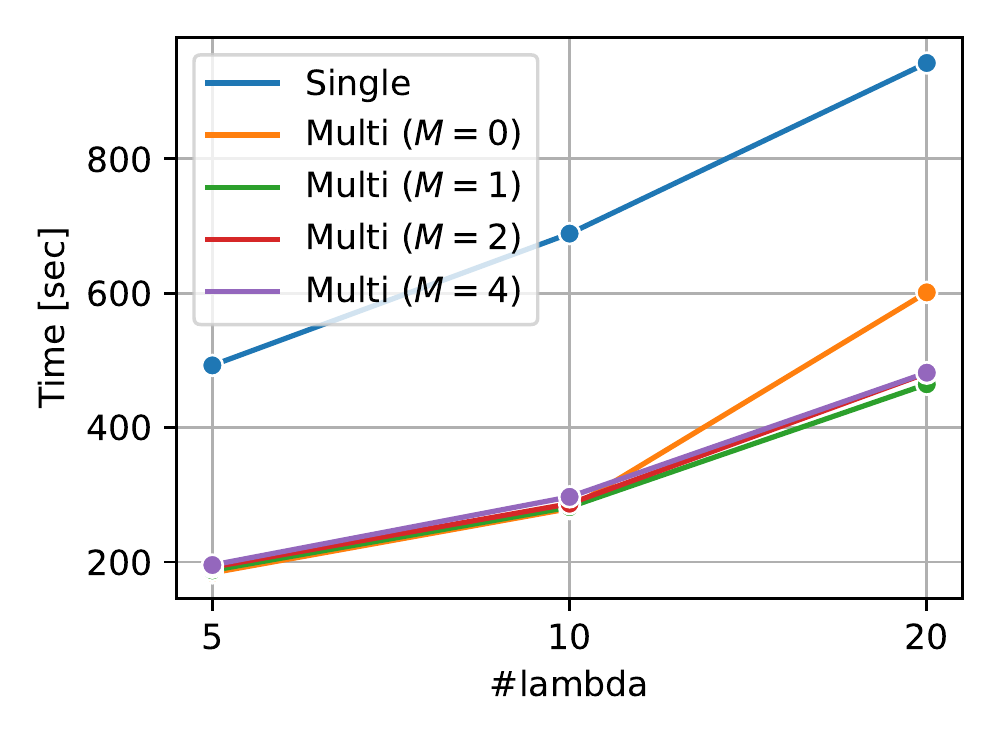}} \\
 \subfloat[][splice]{\includegraphics[width=0.33\hsize]{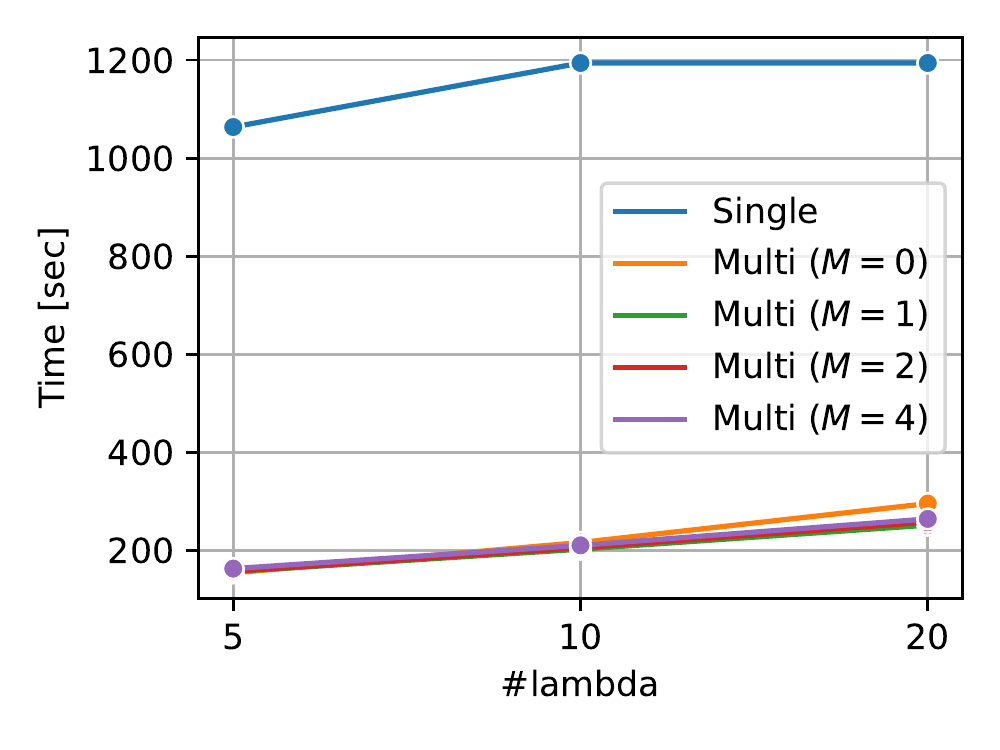}}
 \subfloat[][w1a]{\includegraphics[width=0.33\hsize]{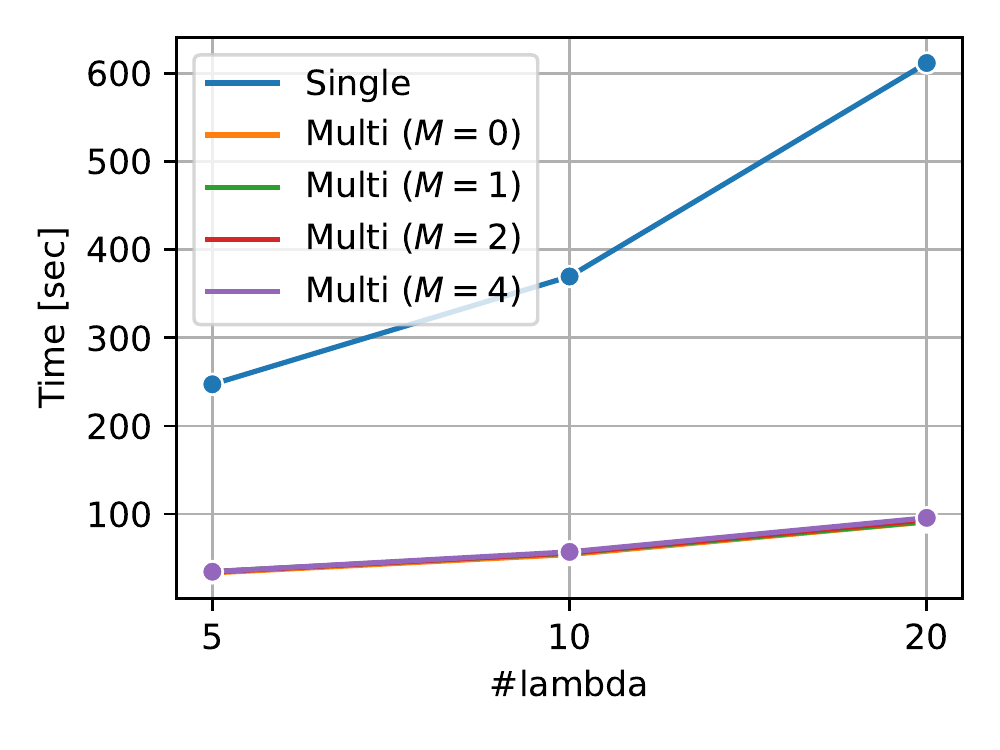}}
 \subfloat[][rhodopsin]{\includegraphics[width=0.33\hsize]{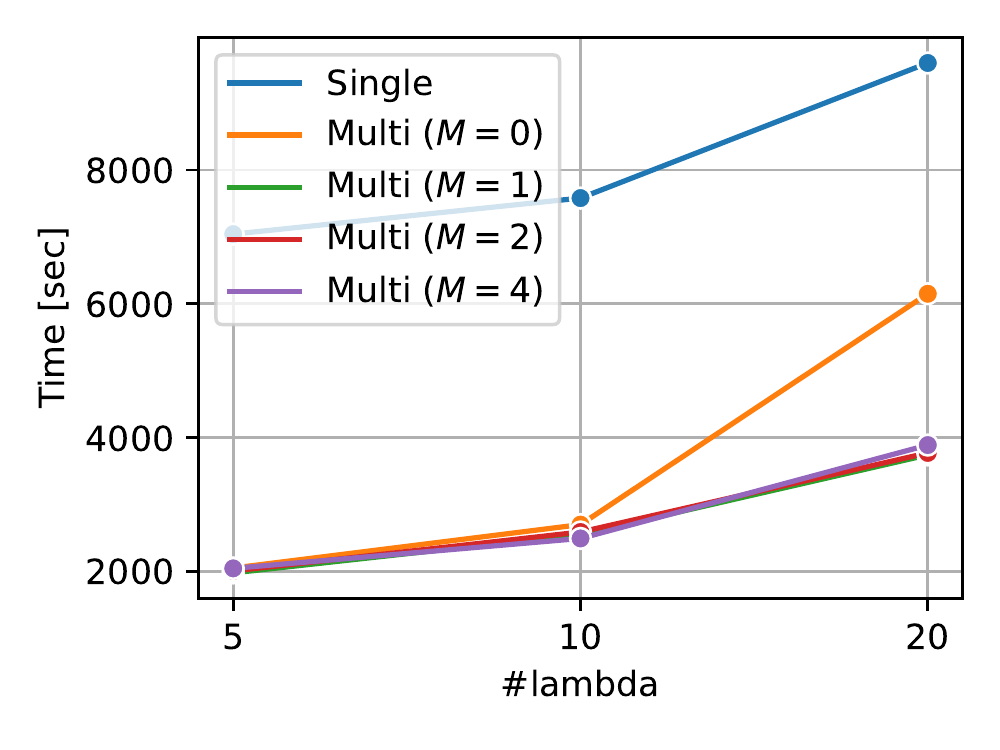}}
 \caption{The computation time for leave-one-out cross-validation for each dataset.
 The use of multiple solutions is effective for all the cases. 
 The Multi-dynamic screening is effective in settings where the number of $\lambda$ is large.
 }
 \label{fig:result_cv}
\end{figure}

\clearpage

\begin{figure}[p]
    \centering
    \includegraphics[width=\hsize]{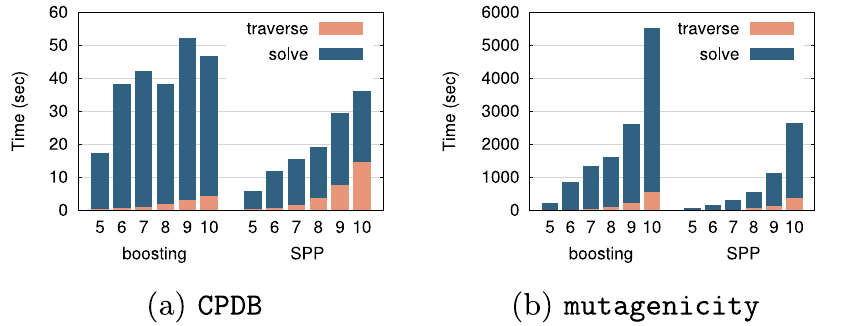}\\
    \includegraphics[width=\hsize]{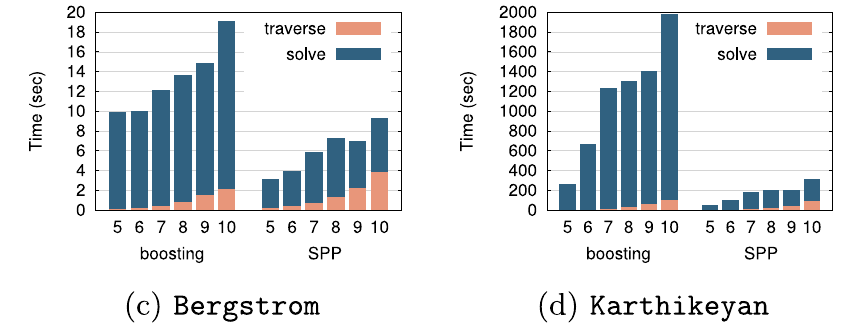}
    \caption{Computational time comparison for graph classification and regression.
    The horizontal axis represents the maximum length of patterns that are mined. Each bar contains computational time taken in the tree traverse (traverse) and the optimization procedure (solve) respectively.}
    \label{fig:time_graph_mining}
\end{figure}

\clearpage

\begin{figure}[p]
    \centering
    \includegraphics[width=\hsize]{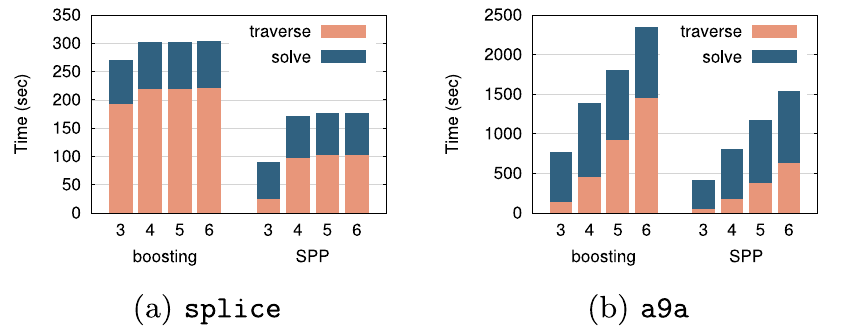}\\
    \includegraphics[width=\hsize]{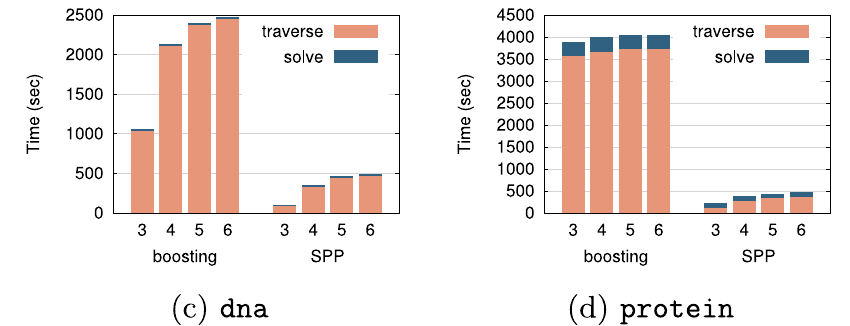}
    \caption{Computational time comparison for item-set classification and regression.
    The horizontal axis represents the maximum length of patterns that are mined. Each bar contains computational time taken in the tree traverse (traverse) and the optimization procedure (solve) respectively.}
    \label{fig:time_itemset_mining}
\end{figure}

\clearpage
\appendix

\clearpage

\section*{Note S1: Proof of Lemma \ref{lem:safe_screening_rule}}

Proofs of the following lemmas are similar to those in \citep{fercoq2015mind,ndiaye2017gap}, however, we explicitly state the proof for our setup.

Before proving Lemma \ref{lem:safe_screening_rule}, we prove two other lemmas: Lemmas \ref{lem:range_dual_optimal} and \ref{lem:range_inner_product}.

\begin{defi}
For a convex function $f: \mathbb{R}^n\to\mathbb{R}$ and a vector $\bm v\in\mathbb{R}^n$,
$\bm g\in\mathbb{R}^n$ is called a {\em subderivative} of $f$ at $\bm v$
if the following condition is met:
\begin{equation*}
\forall \bm z:\quad f(\bm z) - f(\bm v) \geq \bm g^\top(\bm z - \bm v).
\end{equation*}
We denote by $\partial f(\bm v)$ the set of all subderivatives of $f$ at $\bm v$, since such $\bm g$ may not be unique.
\end{defi}

\begin{lemm}
    \label{lem:range_dual_optimal}
    For any pair of feasible solutions $(\bzero, \azero)$, the dual optimal solution $\bm\alpha^*$ is within the intersection of a $\mathbb R^n$-ball $B(\bzero, \azero)$ and a hyperplane $H$ such that
    \begin{equation}
        \begin{split}
            B(\bzero, \azero) &\coloneqq\{\bm\alpha\in\mathbb R^n\mid \|\bm\alpha - \azero\|_2 \le r(\bzero, \azero)\},\\
            H & \coloneqq \{\bm\alpha\in\mathbb R^n\mid \bm\alpha^\top\bm 1 = 0\},
        \end{split}
    \end{equation}
    where $r(\bzero, \azero)$ and $\gamma$ are defined in Lemma \ref{lem:safe_screening_rule}.
\end{lemm}

\begin{proof}
The constraint of $H$ is just derived from \eqref{eq:dual_problem}, so we focus on $B$ in the remainder of the proof.

First, since we assume that the loss function $L$ is $\gamma$-Lipschitz continuous, 
$L^*(\bm\alpha)$ is known to be $(1/\gamma)$-{\em strongly convex} and therefore $D(\bm\alpha)$ is $(1/\gamma)$-{\em strongly concave}, that is, 
the following holds:
\begin{equation}
D(\bm{v}) - D(\bm{u}) \geq \bm g^\top (\bm{v} - \bm{u}) + \frac{1}{2\gamma}\|\bm{v} - \bm{u}\|_2^2 \label{eq:strongly-concave}
\end{equation}
for any $\bm{u}, \bm{v}\in\mathbb{R}^n$
and any subderivative $\bm g \in \partial D(\bm v)$.
See, for example, Section 12.H in \citep{rockafellar2009variational}.

Substituting $\bm{u}\gets\azero$ and $\bm{v}\gets\bm\alpha^*$, we have
\begin{align}
D(\bm\alpha^*) - D(\azero) &\geq \bm g^\top (\bm\alpha^* - \azero) + \frac{1}{2\gamma}\|\bm\alpha^* - \azero\|_2^2 \nonumber \\
&\geq \frac{1}{2\gamma}\|\bm\alpha^* - \azero\|_2^2. \label{eq:strongly-concave-bound}
\end{align}
Here, \eqref{eq:strongly-concave-bound} is obtained by the fact that $\bm g^\top (\bm\alpha^* - \azero) \geq 0$ (See, for example, Proposition B.24 in \citep{bertsekas1999nonlinear}).

Moreover, by the calculation of the dual problem by Fenchel's duality theorem, $P(\bm\beta^*) = D(\bm\alpha^*)$ holds for \eqref{eq:primal_problem} and \eqref{eq:dual_problem} (called the {\em strong duality}; See, for example, Sections 11.H and 11.I of \citep{rockafellar2009variational}).
Therefore, from \eqref{eq:strongly-concave-bound} we have
\begin{align*}
\frac{1}{2\gamma}\|\bm\alpha^* - \azero\|_2^2
	& \leq D(\bm\alpha^*) - D(\azero)
	=  P(\bm\beta^*) - D(\azero) \\
& \leq P(\bzero) - D(\azero) \quad(\because~\bm\beta^*~\text{is the minimizer of}~P).
\end{align*}
This proves the lemma.
\end{proof}

\begin{lemm} \label{lem:range_inner_product}
Under the condition of the dual optimal solution $\bm \alpha^*$ in Lemma \ref{lem:range_dual_optimal},
we can represent an upper bound of $|\Xj^\top\bm\alpha^*|$ as 
\begin{equation}
	|\Xj^\top\bm\alpha^*|
    \leq \max_{\bm\alpha\in B(\bzero, \azero)\cap H} |\Xj^\top\bm\alpha|
    = |\Xj^\top\bm\azero| + r(\bzero, \azero)\|\Xj - \Pi_{\bm 1}(\Xj)\|_2, \label{eq:range_inner_product}
\end{equation}
where $\Pi_{\bm u}(\bm v)$ is defined in Lemma \ref{lem:safe_screening_rule}.
\end{lemm}

\begin{proof}
We prove this via the method of Lagrange multiplier, that is,
\begin{align}
& \max_{\bm\alpha\in B(\bzero, \azero)\cap H} \Xj^\top\bm\alpha
	= \max_{\bm{\alpha}\in\mathbb{R}^n, \xi_1\in\mathbb{R}, \xi_2\in\mathbb{R}} {\cal L}(\bm{\alpha}, \xi_1, \xi_2), \nonumber \\
& \text{where}\qquad 
	{\cal L}(\bm{\alpha}, \xi_1, \xi_2)
	= \Xj^\top\bm\alpha
	- \xi_1\left(\|\bm\alpha - \azero\|_2^2 - r(\bzero, \azero)^2 \right)
	- \xi_2 \bm\alpha^\top\bm 1.
\end{align}
Let
$(\bm\alpha^\#, \xi_1^\#, \xi_2^\#) := \arg\max_{\bm{\alpha}\in\mathbb{R}^n, \xi_1\in\mathbb{R}, \xi_2\in\mathbb{R}} {\cal L}(\bm{\alpha}, \xi_1, \xi_2)$
be the optimal solution of the latter maximization.
Then it is known to satisfy the following conditions ({\em Karush-Kuhn-Tucker condition}):
\begin{align}
& \nabla L(\bm{\alpha}, \xi_1, \xi_2)|_{\bm\alpha^\#, \xi_1^\#, \xi_2^\#} = \bm 0 \label{eq:lagr-derivative}\\
& \xi_1^\# \geq 0 \label{eq:lagr-xi1}\\
& \|\bm\alpha^\# - \azero\|_2^2 - r(\bzero, \azero)^2 \leq 0 \label{eq:lagr-ball}\\
& \xi_1^\# ( \|\bm\alpha^\# - \azero\|_2^2 - r(\bzero, \azero)^2 ) = 0 \label{eq:lagr-ball-xi1}\\
& \xi_2^\# \geq 0 \label{eq:lagr-xi2} \\
& \bm\alpha^{\#\top}\bm 1 = 0 \label{eq:lagr-plane}
\end{align}
In this setup, it is clear that the maximization of ${\cal L}$ must be infinite if $\xi_1 = 0$.
So, assuming $\xi_1^\#\neq 0$ in \eqref{eq:lagr-xi1}, \eqref{eq:lagr-ball} and \eqref{eq:lagr-ball-xi1}, we have
\begin{align}
\|\bm\alpha^\# - \azero\|_2^2 - r(\bzero, \azero)^2 = 0. \label{eq:lagr-ball-xi1-eq}
\end{align}

In addition, \eqref{eq:lagr-derivative} is computed as follows:
\begin{align}
& \nabla L(\bm{\alpha}, \xi_1, \xi_2)|_{\bm\alpha^\#, \xi_1^\#, \xi_2^\#} = \Xj - 2 \xi_1^\# (\bm\alpha^\# - \azero) - \xi_2^\# \bm 1 = \bm 0 \label{eq:lagr-derivative-calc} \\
& \Xj^\top\bm 1 - 2 \xi_1^\# (\bm\alpha^\# - \azero)^\top\bm 1 - \xi_2^\# \|\bm 1\|_2^2 = 0 \nonumber\\
& \Xj^\top\bm 1 = \xi_2^\# \|\bm 1\|_2^2 \qquad (\because \bm\alpha^{\#\top}\bm 1 = \azero^\top\bm 1 = 0) \nonumber\\
& \therefore \xi_2^\# = \frac{\Xj^\top\bm 1}{\|\bm 1\|_2^2}. \nonumber
\end{align}
Moreover, substituting $\xi_2^\#$ in \eqref{eq:lagr-derivative-calc} we have
\begin{align}
& \Xj - 2 \xi_1^\# (\bm\alpha^\# - \azero) - \frac{\Xj^\top\bm 1}{\|\bm 1\|_2^2} \bm 1
	= \Xj - 2 \xi_1^\# (\bm\alpha^\# - \azero) - \Pi_{\bm 1}(\Xj) = \bm 0 \label{eq:lagr-derivative-xi2}\\
& 4 (\xi_1^\#)^2 \|\bm\alpha^\# - \azero\|_2^2 = 4 (\xi_1^\#)^2 r(\bzero, \azero)^2 = \|\Xj - \Pi_{\bm 1}(\Xj)\|_2^2 \qquad(\because \eqref{eq:lagr-ball-xi1-eq}) \nonumber\\
& \therefore \xi_1^\# = \frac{\|\Xj - \Pi_{\bm 1}(\Xj)\|_2}{2 r(\bzero, \azero)}. \nonumber
\end{align}
As a result, substituting $\xi_1^\#$ in \eqref{eq:lagr-derivative-xi2} we have
\begin{align*}
& \bm\alpha^\# = \azero + r(\bzero, \azero) \frac{\Xj - \Pi_{\bm 1}(\Xj)}{\|\Xj - \Pi_{\bm 1}(\Xj)\|_2},
\end{align*}
and
\begin{align}
& \max_{\bm\alpha\in B(\bzero, \azero)\cap H} \Xj^\top\bm\alpha = L(\bm{\alpha}^\#, \xi_1^\#, \xi_2^\#) = \Xj^\top\bm\alpha^\# \nonumber\\
& = [\Xj - \Pi_{\bm 1}(\Xj)]^\top \bm\alpha^\# \qquad(\because [\Pi_{\bm 1}(\Xj)]^\top\bm\alpha^\# = 0~\text{by}~\eqref{eq:lagr-plane}) \nonumber\\
& = [\Xj - \Pi_{\bm 1}(\Xj)]^\top\azero + r(\bzero, \azero) \|\Xj - \Pi_{\bm 1}(\Xj)\|_2 \nonumber\\
& = \Xj^\top\azero + r(\bzero, \azero) \|\Xj - \Pi_{\bm 1}(\Xj)\|_2. \label{eq:lagr-max-pos-innerproduct}
\end{align}

This concludes $\max_{\bm\alpha\in B(\bzero, \azero)\cap H} \Xj^\top\bm\alpha = \Xj^\top\azero + r(\bzero, \azero) \|\Xj - \Pi_{\bm 1}(\Xj)\|_2$. The result consequently proves that
\begin{align*}
& \max_{\bm\alpha\in B(\bzero, \azero)\cap H} (-\Xj)^\top\bm\alpha \\
& = -\Xj^\top\azero + r(\bzero, \azero) \|-\Xj - \Pi_{\bm 1}(-\Xj)\|_2 \quad(\because~\Xj\gets(-\Xj)~\text{in}~\eqref{eq:lagr-max-pos-innerproduct}) \\
& = -\Xj^\top\azero + r(\bzero, \azero) \|\Xj - \Pi_{\bm 1}(\Xj)\|_2.
\end{align*}
Combining them, we have $\max_{\bm\alpha\in B(\bzero, \azero)\cap H} |\Xj^\top\bm\alpha| = |\Xj^\top\azero| + r(\bzero, \azero) \|\Xj - \Pi_{\bm 1}(\Xj)\|_2$.
\end{proof}

Finally, Lemma \ref{lem:safe_screening_rule} is proved as follows:

\begin{proof}[Proof of Lemma \ref{lem:safe_screening_rule}]
Suppose that $u_j(\bzero, \azero) < \lambda_1$. Then, by Lemma \ref{lem:range_inner_product} we have
\begin{align*}
& \lambda > u_j(\bzero, \azero) := |\Xj^\top\azero| + r(\bzero, \azero) \|\Xj - \Pi_{\bm 1}(\Xj)\|_2 \\
& \geq \max_{\bm\alpha\in B(\bzero, \azero)\cap H} |\Xj^\top\bm\alpha| \geq |\Xj^\top\bm\alpha^*|.
\end{align*}
By equation \eqref{eq:optimality}, $\beta^*_j = 0$ must hold.
\end{proof}

%%%%%%%%%%%%%%%%%%%%%%%%%%%%%%%%

\section*{Note S2: Proof of Theorem \ref{the:safe_pattern_pruning_rule}}

\begin{proof}
 First, we prove that the SPP-score in \eq{eq:spp_score} is greater than or equal to the safe screening score in \eq{eq:safe_screening_score}, i.e.,  $v_j(\bzero, \azero) \ge u_j(\bzero, \azero)$. 
 This can be shown as
 \begin{align*}
  v_j(\tilde{\bm \beta}, \tilde{\bm \alpha})
  -
  u_j(\tilde{\bm \beta}, \tilde{\bm \alpha})
  &
  =
  \max
  \left\{ 
  \sum_{i:\tilde{\alpha}_i > 0}x_{ij}\tilde{\alpha}_i,
  -\sum_{i:\tilde{\alpha}_i < 0}x_{ij}\tilde{\alpha}_i
  \right\}
  - 
  |\bm x_{:j}^\top \tilde{\bm \alpha}| + r(\bzero, \azero)\left(\|\Xj\|_2 - \|\Xj-\Pi_{\bm 1}(\Xj)\|_2\right)
  \\
  &
  \ge
  \max
  \left\{ 
  \sum_{i:\tilde{\alpha}_i > 0}x_{ij}\tilde{\alpha}_i,
  -\sum_{i:\tilde{\alpha}_i < 0}x_{ij}\tilde{\alpha}_i
  \right\}
  - 
   \max
   \left\{
    \sum_{i \in [n]} x_{ij} \alpha_i,
    -\sum_{i \in [n]} x_{ij} \alpha_i
   \right\}
  \\
  &
  \ge 0.
 \end{align*}
 Therefore, using Lemma~\ref{lem:safe_screening_rule}, we have
 \[
 v_j(\bzero, \azero) < \lambda \Rightarrow u_j(\bzero, \azero) < \lambda \Rightarrow \beta_j^*= 0.
 \]
 Next, we prove that, for pair of pattern $p_j$ and $p_k$ such that $p_k \sqsubset p_j$, the SPP score of $p_j$ is grater than or equal to that of $p_k$, i.e., $v_j(\tilde{\bm \beta}, \tilde{\bm \alpha}) \ge v_k(\tilde{\bm \beta}, \tilde{\bm \alpha})$.
 To show this, we prove that each of the two terms of the SPP score satisfies the intended inequality relationship, thereby showing that the SPP score as a whole also satisfies the inequality relationship.
 The inequality relationship for the first term of the SPP score is shown as follows.
 From Lemma~\ref{lem:monotnicity}, it is clear that
 \begin{align*}
  \sum_{i:\tilde{\alpha}_i > 0}x_{ij}\tilde{\alpha}_i \ge \sum_{i:\tilde{\alpha}_i > 0}x_{ik}\tilde{\alpha}_i.
 \end{align*}
 Therefore, we have
 \begin{equation}        
  \max
  \left\{ 
   \sum_{i:\tilde{\alpha}_i > 0}x_{ij}\tilde{\alpha}_i,
   -\sum_{i:\tilde{\alpha}_i < 0}x_{ij}\tilde{\alpha}_i
  \right\}
  \ge 
  \max
  \left\{ 
   \sum_{i:\tilde{\alpha}_i > 0}x_{ik}\tilde{\alpha}_i,
   -\sum_{i:\tilde{\alpha}_i < 0}x_{ik}\tilde{\alpha}_i
  \right\}.
  \label{eq:pruning_criteria_first}
 \end{equation}
 The inequality relationship for the second term of the SPP score is easily shown by noting that 
 \begin{align*}
  \|\bm x_{:j}\|_2 \ge \|\bm x_{:k}\|_2.
 \end{align*}
 This means that 
 \begin{align*}
  v_j(\tilde{\bm \beta},  \tilde{\bm \alpha}) < \lambda
  ~\Rightarrow~
  v_k(\tilde{\bm \beta},  \tilde{\bm \alpha}) < \lambda
  ~\Rightarrow~
  \beta^*_k = 0
  ~
  ~\forall k \in [d] \text{ s.t. } p_k \sqsubset p_j. 
 \end{align*}
\end{proof}

%%%%%%%%%%%%%%%%%%%%%%%%%%%%%%%%%%%%%%%%

\section*{Note S3: Proof of Theorem \ref{the:multiple_safe_screening_rule}}

In order to prove the theorem, we first prove the following lemma.

\begin{lemm}[Union of two hyperspheres]
    \label{lem:intersection_of_two}
    Suppose that two hyperspheres in $\mathbb R^n$, denoted by
    $S_1 = \{\bm v\in\mathbb R^n\mid\|\bm v - \bm c_1\|_2 = r_1\}$ and
    $S_2 = \{\bm v\in\mathbb R^n\mid\|\bm v - \bm c_2\|_2 = r_2\}$,
    satisfies $S_1\cap S_2\not\equiv\emptyset$ and $S_1 \not\equiv S_2$, that is,
    \begin{align}
	& \delta := \|\bm c_1 - \bm c_2\|_2 > 0, \label{eq:union-not-concentric} \\
	& r_1 + r_2 \geq \delta, \label{eq:union-no-intersection} \\
	& |r_1 - r_2| \leq \delta. \label{eq:union-inclusive}
	\end{align}
    Then, the intersection of them $S_1\cap S_2$ is identical to the intersection $S^\prime\cap H^\prime$ of
    the following hypersphere $S^\prime$ and hyperplane $H^\prime$:
    \begin{align*}
        S^\prime &= \{\bm v\in\mathbb R^n\mid \|\bm v - \bm c^\prime\|_2 = r^\prime\},\\
        H^\prime &= \{\bm v\in\mathbb R^n\mid (\bm v - \bm c^\prime)^\top(\bm c_1 - \bm c_2) = 0\},
    \end{align*}
    where $\bm c^\prime$ (center of $S^\prime$) and $r^\prime$ (radius of $S^\prime$) are defined as follows:
    \begin{align*}
        \bm c^\prime &= t\bm c_1 + (1-t)\bm c_2,\\
        r^\prime &= \sqrt{r_2^2 - t^2 \delta^2}, \\
        t &= \frac{1}{2}\left(1 + \frac{r_2^2 - r_1^2}{\delta^2}\right).\\
    \end{align*}
    \label{lem:intersection_sphere}
\end{lemm}

\begin{proof}[Proof of Lemma \ref{lem:intersection_of_two}]
	Let $E: \mathbb{R}^n\to\mathbb{R}^n$ be an distance-preserving mapping such that
    \begin{align*}
        E\bm c_1 &= \bm 0,\\
        E\bm c_2 &= [\delta, \underbrace{0,\ldots, 0}_{n-1}]^\top,
    \end{align*}
	where $\delta = \|\bm c_1 - \bm c_2\|_2$.
	Note that such a mapping can be obtained as follows:
	\begin{itemize}
	\item Let $E\bm v := \Theta(\bm v - \bm c_1)$ ($\Theta\in\mathbb{R}^{n\times n}$).
	\item Set $\Theta_{1:} = \frac{1}{\delta^2}(\bm c_1 - \bm c_2)$.
	\item Set other rows of $\Theta$ so that $\Theta$ is an orthogonal matrix. This can be done by Gram-Schmidt algorithm.
	\end{itemize}

	Let $\bm v\in S_1 \cap S_2$, and $\bm v^\prime$ be
    \[
        E\bm v = \bm v^\prime = [v_1^\prime,\ldots,v_n^\prime]^\top.
    \]
    Then, since $E$ is distance-preserving, we have
    \begin{align}
        & \|E\bm v- E\bm c_1\|_2^2 = r_1^2 \Longleftrightarrow \sum_{i=1}^n v_i^{\prime2} = r_1^2, \label{eq:equal_length_1} \\
        & \|E\bm v - E\bm c_2\|_2^2 = r_2^2 \Longleftrightarrow \sum_{i=2}^n v_i^{\prime2} + (v_1^\prime - \delta)^2 = r_2^2. \label{eq:equal_length_2}
    \end{align}
    Taking the difference between equations (\ref{eq:equal_length_1}) and (\ref{eq:equal_length_2}), we have
    \begin{align}
	& v_1^\prime = \frac{r_1^2 - r_2^2 + \delta^2}{2\delta}, \label{eq:union-v1} \\
	& \sum_{i=2}^n v_i^{\prime2}
		= r_2^2 - (v_1^\prime  - \delta)^2
		= r_2^2 - \left(\frac{r_1^2 - r_2^2 - \delta^2}{2\delta} \right)^2
		= r_2^2 - t^2 \delta^2. \label{eq:union-v2}
	\end{align}
	Note that the value \eqref{eq:union-v2} is nonnegative because
    \begin{align*}
	r_2^2 - \left(\frac{r_1^2 - r_2^2 - \delta^2}{2\delta} \right)^2
	&= \frac{1}{2\delta}(2\delta r_2 + r_1^2 - r_2^2 - \delta^2)(2\delta r_2 - r_1^2 + r_2^2 + \delta^2) \\
	&= \frac{1}{2\delta}[r_1^2 - (r_2 - \delta)^2][(r_2 + \delta)^2 - r_1^2] \\
	&= \frac{1}{2\delta}(\underbrace{r_1 + r_2 - \delta}_{\geq 0~\because\text{\eqref{eq:union-no-intersection}}})(\underbrace{r_1 - r_2 + \delta}_{\geq 0~\because\text{\eqref{eq:union-inclusive}}})(\underbrace{r_1 + r_2 + \delta}_{> 0~\text{clearly}})(\underbrace{-r_1 + r_2 + \delta}_{\geq 0~\because\text{\eqref{eq:union-inclusive}}}).
	\end{align*}
	In summary, equation \eqref{eq:union-v1} implies that $v_1^\prime$ of $\bm v^\prime$ is constant.
    In addition, equation \eqref{eq:union-v2} implies that
    $[v_2^\prime,\ldots,v_n^\prime]$ is on a hypersphere whose center is
    $[v_2^\prime,\ldots,v_n^\prime]^\top = [\underbrace{0,\ldots, 0}_{n-1}]^\top$.
    So, if we take a hypersphere
    whose center is $E\bm c^\prime = [v_1^\prime, \underbrace{0,\ldots, 0}_{n-1}]^\top$ and
    whose radius is $r^\prime = \sqrt{r_2^2 - t^2 \delta^2}$,
    then the intersection of it and the hyperplane $v_1^\prime = \frac{r_1^2 - r_2^2 + \delta^2}{2\delta}$
    composes $\bm v^\prime$.
    
    Finally we derive the center of the hypersphere in the original space $\bm c^\prime$.
    (Note that $r^\prime$ is the same between in the original space and the space after applying $E$,
    since $E$ is distance-preserving.)
    Again, in the space after applying $E$,
    \begin{align*}
	E\bm c_1 = [0, \underbrace{0,\ldots, 0}_{n-1}]^\top,
	\quad
	E\bm c^\prime = [v_1^\prime, \underbrace{0,\ldots, 0}_{n-1}]^\top,
	\quad
	E\bm c_2 = [\delta, \underbrace{0,\ldots, 0}_{n-1}]^\top.
	\end{align*}
	Since $E$ is distance-preserving, $\bm c^\prime$ in the original space can be computed as
    \begin{align*}
	\bm c^\prime &= \bm c_1 + \frac{v_1^\prime}{\delta} (\bm c_2 - \bm c_1) \\
	&= \bm c_1 + \frac{r_1^2 - r_2^2 + \delta^2}{2\delta^2} (\bm c_2 - \bm c_1) \\
	&= \frac{\delta^2 - r_1^2 + r_2^2}{2\delta^2} \bm c_1 + \left( 1 - \frac{\delta^2 - r_1^2 + r_2^2}{2\delta^2} \right) \bm c_2 = t\bm c_1 + (1-t)\bm c_2.
	\end{align*}
	This derives $\bm c^\prime$ and $t$ in the lemma.
\end{proof}

\begin{proof}[Proof of Theorem \ref{the:multiple_safe_screening_rule}]
	In order to prove
    \[
        \max_{\bm\alpha\in B_1\cap B_2\cap H} |\Xj^\top\bm\alpha| = \max\{u_j^+, u_j^-\},
    \]
    first we note that
    \[
        \max_{\bm\alpha} |\Xj^\top\bm\alpha| = \max\{\max_{\bm\alpha}\Xj^\top\bm\alpha, \max_{\bm\alpha}(-\Xj)^\top\bm\alpha\}.
    \]
	The second expression can be obtained by just replacing $\Xj$ with $-\Xj$
	in the first expression. So we discuss only the first expression, that is,
    \begin{equation}
        \max_{\bm\alpha\in B_1\cap B_2\cap H}\Xj^\top\bm\alpha.
        \label{eq:multi_screening_maximum}
    \end{equation}
    
    The Lagrangian function of \eqref{eq:multi_screening_maximum} is defined as
    \[
        \mathcal L(\bm \alpha, \xi_1, \xi_2, \xi_3) = \Xj^\top\bm\alpha
        - \xi_1(\|\bm\alpha - \tilde{\bm\alpha}^{(1)}\|_2^2 - r(R_1)^2)
        - \xi_2(\|\bm\alpha - \tilde{\bm\alpha}^{(2)}\|_2^2 - r(R_2)^2)
        - \xi_3\bm\alpha^\top\bm 1.
    \]
    Then, the optimal solution of \eqref{eq:multi_screening_maximum}, denoted by $\tilde{\bm\alpha}^*$, must satisfy the following conditions:
    \begin{align}
        \nabla_{\bm\alpha}\mathcal L|_{\bm\alpha=\tilde{\bm\alpha}^*} &= \bm 0,\label{eq:optimal_cond_grad}\\
        \xi_1 &\ge 0,\\
        \|\tilde{\bm\alpha}^* - \tilde{\bm\alpha}^{(1)}\|_2^2 - r(R_1)^2 &\le 0,\label{eq:cond_first_feasible_ball}\\
        \xi_1(\|\tilde{\bm\alpha}^* - \tilde{\bm\alpha}^{(1)}\|_2^2 - r(R_1)^2) &= 0,\\
        \xi_2 &\ge 0,\\
        \|\tilde{\bm\alpha}^* - \tilde{\bm\alpha}^{(2)}\|_2^2 - r(R_2)^2 &\le 0,\label{eq:cond_second_feasible_ball}\\
        \xi_2(\|\tilde{\bm\alpha}^* - \tilde{\bm\alpha}^{(2)}\|_2^2 - r(R_2)^2) &= 0,\\
        \bm\alpha^\top\bm 1 &= 0.\label{eq:cond_hyperplane}
    \end{align}
    Note that, if $\xi_1=\xi_2=0$, (i.e., neither \eqref{eq:cond_first_feasible_ball} nor \eqref{eq:cond_second_feasible_ball} are active), then $\max_{\bm\alpha} \mathcal L$ cannot be finite.
    So we can assume that $(\xi_1, \xi_2) \neq (0, 0)$.
    From \eqref{eq:optimal_cond_grad}, we have
    \[
        \nabla_{\bm\alpha} \mathcal L = \Xj 
        - 2\xi_1(\bm\alpha - \tilde{\bm\alpha}^{(1)}) 
        - 2\xi_2(\bm\alpha - \tilde{\bm\alpha}^{(2)})
        - \xi_3\bm 1.
    \]
    Since $\xi_1+\xi_2\neq 0$, we have
    \[
        \tilde{\bm\alpha}^* = \frac{1}{2(\xi_1 + \xi_2)}\left(\Xj 
        + 2\xi_1\tilde{\bm\alpha}^{(1)}
        + 2\xi_2\tilde{\bm\alpha}^{(2)}
        - \xi_3\bm 1\right).
    \]
    Moreover, since $\bm 1^\top\tilde{\bm\alpha}^{(1)}=\bm 1^\top\tilde{\bm\alpha}^{(2)}=0$ from \eqref{eq:cond_hyperplane}, we have
    \begin{align*}
        \frac{1}{2(\xi_1+\xi_2)}\left(\Xj^\top\bm 1 - \xi_3\bm 1^\top\bm 1\right) =0\\
        \therefore \xi_3 = \frac{\Xj^\top\bm 1}{\bm 1^\top\bm 1}.
    \end{align*}
	
	Here we calculate the solution $\tilde{\bm\alpha}^*$ based on the values of $\xi_1$ and $\xi_2$.
	First, if $\xi_1\neq 0$ and $\xi_2 = 0$, (i.e., \eqref{eq:cond_first_feasible_ball} is active but not \eqref{eq:cond_second_feasible_ball}), then we have $\|\bm\alpha - \tilde{\bm\alpha}^{(1)}\|_2^2=r(R_1)^2$ and
    \begin{align*}
        \frac{1}{4\xi_1^2}\|\Xj - \Pi_{\bm 1}(\Xj)\|_2^2 = r(R_1)^2\\
        \therefore \xi_1 = \frac{\|\Xj - \Pi_{\bm 1}(\Xj)\|_2}{r(R_1)},
    \end{align*}
    then
    \[
        \tilde{\bm\alpha}^* = \frac{r(R_1)}{\|\Xj-\Pi_{\bm 1}(\Xj)\|_2}\left(\Xj - \Pi_{\bm 1}(\Xj)\right) + \tilde{\bm\alpha}^{(1)},
    \]
    and the maximized result is calculated as
    \begin{align}
        \mathcal L(\tilde{\bm\alpha}^*, \xi_1, \xi_2, \xi_3) &= \Xj^\top\tilde{\bm\alpha}^*\notag\\
        &= (\Xj - \Pi_{\bm 1}(\Xj))^\top\tilde{\bm\alpha}^*\notag\\
        &= \Xj^\top\tilde{\bm\alpha}^{(1)} + r(R_1)\|\Xj - \Pi_{\bm 1}(\Xj)\|_2. \label{eq:maximum_first_feasible}
    \end{align}
	In this case $\tilde{\bm\alpha}^* \in B_2$ must hold, that is, $\|\tilde{\bm\alpha}^*-\tilde{\bm\alpha}^{(2)}\|_2^2 \le r(R_2)^2$. So we have
    \[
        \frac{\Xj^\top\bm \delta}{\|\Xj - \Pi_{\bm 1}(\Xj)\|_2} \le \frac{r(R_2)^2 - r(R_1)^2 - \|\bm \delta\|_2^2}{2r(R_1)}.
    \]

	If $\xi_1 = 0$ and $\xi_2\neq 0$, the calculation can be done similarly: we can conclude that
    \[
        \tilde{\bm\alpha}^* = \frac{r(R_2)}{\|\Xj-\Pi_{\bm 1}(\Xj)\|_2}\left(\Xj - \Pi_{\bm 1}(\Xj)\right) + \tilde{\bm\alpha}^{(2)}
    \]
    and
    \begin{equation}
        \mathcal L(\tilde{\bm\alpha}^*, \xi_1, \xi_2, \xi_3) = \Xj^\top\tilde{\bm\alpha}^{(2)} + r(R_2)\|\Xj - \Pi_{\bm 1}(\Xj)\|_2. \label{eq:maximum_second_feasible}
    \end{equation}
	Since $\tilde{\bm\alpha}^*\in B_1$, we also have
    \[
        \frac{\Xj^\top\bm \delta}{\|\Xj - \Pi_{\bm 1}(\Xj)\|_2} \ge \frac{r(R_2)^2 - r(R_1)^2 + \|\bm \delta\|_2^2}{2r(R_2)}.
    \]

	Finally we show the case of $\xi_1\neq 0$ and $\xi_2\neq 0$.
	In this case, since $\|\bm\alpha-\tilde{\bm\alpha}^{(1)}\|_2^2 = r(R_1)^2$ and
	$\|\bm\alpha-\tilde{\bm\alpha}^{(2)}\|_2^2 = r(R_2)^2$,
	the constraint can be represented as an intersection of two hyperspheres in $\mathbb R^n$.
	So we replace them with Lemma \ref{lem:intersection_of_two}.
	Let $S_1$ and $S_2$ be the surfaces of $B_1$ and $B_2$, respectively. Then the problem is rewritten as:
	\[
        \max_{\bm\alpha\in S_1\cap S_2 \cap H}\Xj^\top\bm\alpha = \max_{\bm\alpha\in S^\prime\cap H^\prime\cap H}\Xj^\top\bm\alpha,
    \]
	where
    \begin{align*}
        S^\prime &= \{\bm v\in\mathbb R^n\mid \|\bm v - \tilde{\bm\alpha}^\prime\| < r^\prime\},\\
        H^\prime &= \{\bm v\in\mathbb R^n\mid (\bm v - \tilde{\bm\alpha}^\prime)^\top \bm \delta = 0\},
    \end{align*}
	and $\tilde{\bm\alpha}^\prime, r^\prime, \bm \delta$ are the ones defined in Theorem \ref{the:multiple_safe_screening_rule}.
	Its Lagrangian function $\mathcal L^\prime$ is defined as
    \[
        \mathcal L^\prime (\bm\alpha, \xi_1^\prime, \xi_2^\prime, \xi_3^\prime) = \Xj^\top\bm\alpha
        - \xi_1^\prime (\|\bm\alpha - \tilde{\bm\alpha}^\prime\|_2^2 - r^{\prime 2})
        - \xi_2^\prime (\bm \alpha - \tilde{\bm\alpha})^\top \bm \delta
        - \xi_3^\prime \bm \alpha^\top\bm 1,
    \]
	with the optimality conditions
    \begin{align}
        \nabla_{\bm\alpha} \mathcal L|_{\bm\alpha=\tilde{\bm\alpha}^*} &= \bm 0,\label{eq:stationary_intersection}\\
        % \xi_1^\prime&\ge 0,\\
        \|\tilde{\bm\alpha}^* - \tilde{\bm\alpha}^\prime\|_2^2 - r^\prime &= 0,\label{eq:sphere_intersection}\\
        % \xi_2^\prime&\ge 0,\\
        (\tilde{\bm\alpha}^* - \tilde{\bm\alpha}^\prime)^\top \bm \delta &= 0,\label{eq:hyperplane_intersection}\\
        % \xi_3^\prime&\ge 0,\\
        \tilde{\bm\alpha}^{*\top}\bm 1&= 0.\label{eq:hyperplane_intersection_feasible}
    \end{align}
	Noticing that $\xi_1\neq 0$ (otherwise $\max_{\bm\alpha} \mathcal L^\prime$ is not bounded),
	from (\ref{eq:stationary_intersection}) we have
    \[
        \nabla\mathcal L = \Xj - 2\xi_1^\prime(\bm\alpha - \tilde{\bm\alpha}^\prime) - \xi_2^\prime\bm \delta - \xi_3^\prime\bm 1
    \]
    and
    \[
        \tilde{\bm\alpha}^* = \tilde{\bm\alpha}^\prime + \frac{1}{2\xi_1^\prime}\left(\xi_2^\prime\bm \delta + \Xj -\xi_3^\prime\bm 1\right).
    \]
	Since $\bm 1^\top\bm \delta = 0$, from (\ref{eq:hyperplane_intersection_feasible}) we have
    \[
        \xi_3^\prime = \frac{\Xj^\top\bm 1}{\|\bm 1\|_2^2}
    \]
    and from (\ref{eq:hyperplane_intersection}) we have
    \[
        \xi_2^\prime = - \frac{\Xj^\top\bm \delta}{\|\bm \delta\|_2^2}.
    \]
	Then, from (\ref{eq:sphere_intersection}) we have
    \[
        \xi_1^\prime = \frac{\|\Xj  - \Pi_{\bm 1}(\Xj)- \Pi_{\bm \delta}(\Xj)\|_2^2}{2{r}^\prime}.
    \]
    Since $\bm\alpha\in H$ and $\bm\alpha\in H^\prime$,
    we have $\bm\alpha^\top\Pi_{\bm 1}(\Xj)=0$ and
    $\bm\alpha^\top\Pi_{\bm \delta}(\Xj) = \tilde{\bm\alpha}^\top\Pi_{\bm \delta}(\Xj)$,
    respectively. Thus we can conclude that
    \begin{align*}
        \mathcal L^*(\tilde{\bm\alpha}^*, \xi_1^\prime, \xi_2^\prime, \xi_3^\prime)
        &= \Xj^\top\tilde{\bm\alpha}^*\\
        &= (\Xj - \Pi_{\bm 1}(\Xj) - \Pi_{\bm \delta}(\Xj))^\top\tilde{\bm\alpha}^* + \tilde{\bm\alpha}^{\prime\top}\Pi_{\bm 1}(\Xj)\\
        &= \Xj^\top\tilde{\bm\alpha}^\prime + r^\prime \|\Xj - \Pi_{\bm 1}(\Xj) - \Pi_{\bm \delta}(\Xj)\|_2^2.
    \end{align*}
\end{proof}

%%%%%%%%%%%%%%%%%%%%%%%%%%%%%%%%

\section*{Note S4: Algorithms}

We present Algorithms \ref{alg:safe_pattern_pruning} to \ref{alg:pathwise_cv_spp} described in this paper.

\begin{algorithm}[p]
    \begin{algorithmic}
        \REQUIRE{$X, \bm y, \lambda, \kappa, R=(\tilde{\bm\beta}, \tilde{\beta}_0, \tilde{\bm\alpha})$}
        \ENSURE{$\mathcal A$}
        \STATE $\mathcal A\leftarrow\emptyset$
        \STATE $\mathcal P\leftarrow\{\emptyset\}$
        \WHILE{$\mathcal P\neq \emptyset$}
            \STATE Pop from the top of $\mathcal P$ as $p$
            \STATE Enumerate expanded patterns $\mathcal P^\prime$ from $p$
            \FOR{$p_j\in\mathcal P^\prime$}
                \IF{$v_j(R) < \lambda$}
                    \STATE comtinue
                \ENDIF
                \IF{$u_j(R) \ge \lambda$}
                    \STATE $\mathcal A\leftarrow\mathcal A\cup\{j\}$
                \ENDIF
                \STATE Push $p^\prime$ into the top of $\mathcal P$ 
            \ENDFOR
        \ENDWHILE
    \end{algorithmic}
    \caption{Safe Pattern Pruning}
    \label{alg:safe_pattern_pruning}
\end{algorithm}

\begin{algorithm}[p]
    \begin{algorithmic}
        \REQUIRE{$X, \bm y, \{(\lambda^{(k)}, \kappa^{(k)})\}_{k\in[K]}, \epsilon$}
        \ENSURE{$\{\bm\beta^{*(k)}, \beta^{*(k)}_0\}_{k\in[K]}$}
        \STATE $\bm\beta, \beta_0\leftarrow \bm 0, 0$
        \FOR{$k\in[K]$}
            \STATE $\lambda \leftarrow \lambda^{(k)}$
            \STATE $\kappa \leftarrow \kappa^{(k)}$
            \STATE Compute $\bm\alpha$ from $\bm\beta, \beta_0$ by the dual scaling
            \STATE $R \leftarrow (\bm\beta, \beta_0, \bm\alpha)$
            \STATE $\mathcal A\leftarrow\mathrm{SafePatternPruning}(X, \bm y, \lambda, \kappa, R)$
            \WHILE{$true$}
                \STATE Update $\bm\beta, \beta_0$ using the sub-gradient of $P$
                \STATE Copute $\bm\alpha$ from $\bm\beta, \beta_0$ by hte dual scaling
                \STATE $R\leftarrow (\bm\beta, \beta_0, \bm\alpha)$
                \IF{$G(R) < \epsilon$}
                    \STATE $\bm\beta^{*(k)}, \beta^{*(k)}_0\leftarrow \bm\beta, \beta_0$
                    \BREAK
                \ENDIF
                \STATE Remove inactive patterns from $\mathcal A$ using safe screening
            \ENDWHILE
        \ENDFOR
    \end{algorithmic}
    \caption{Pathwise optimization with SPP}
    \label{alg:pathwise_spp}
\end{algorithm}

\begin{algorithm}[p]
    \begin{algorithmic}
        \REQUIRE{$X, \bm y, \{\lambda^{(k)}\}_{k\in[K]}, \{\kappa^{(k^\prime)}\}_{k^\prime\in[K^\prime]}, \epsilon, M$}
        \ENSURE{$\{R^{(k, k^\prime)}\}_{(k, k^\prime)\in[K]\times[K^\prime]}$}
        \FOR{$k\in[K]$}
            \FOR{$k^\prime\in[K^\prime]$}
                \STATE $\lambda\leftarrow\lambda^{(k)}$
                \STATE $\kappa\leftarrow\kappa^{(k^\prime)}$
                \STATE $\mathcal R\leftarrow\emptyset$
                \IF{$k > 1$}
                    \STATE $\mathcal R\leftarrow\mathcal R\cup\{R^{(k-1, k^\prime)}\}$
                \ENDIF
                \IF{$k^\prime > 1$}
                    \STATE $\mathcal R\leftarrow\mathcal R\cup\{R^{(k, k^\prime-1)}\}$
                \ENDIF
                \IF{$\mathcal R = \emptyset$}
                    \STATE $\mathcal R\leftarrow\mathcal R\cup\{(\bm 0, 0, \bm 0)\}$
                \ENDIF
                \FOR{$R\in\mathcal R$}
                    \STATE $\bm\beta, \beta_0, \bm\alpha\leftarrow R$
                    \STATE Update $\bm\alpha$ from $\bm\beta, \beta_0$ by dual scaling
                \ENDFOR
                \STATE $\mathcal A \leftarrow\mathrm{MultiSafePatternPruning}(X, \bm y, \lambda, \kappa, \mathcal R)$
                \FOR{$m\in\{1,2,\ldots\}$}
                    \FOR{$R\in\mathcal R$}
                        \STATE $\bm\beta, \beta_0, \bm\alpha\leftarrow R$
                        \STATE Update $\bm\beta, \beta_0$ of $R$ using the sub-gradient of $P$
                        \STATE Update $\bm\alpha$ from $\bm\beta, \beta_0$ using the dual scaling
                        \STATE $R\leftarrow(\bm\beta, \beta_0, \bm\alpha)$
                    \ENDFOR
                    \IF{$\min_{R\in \mathcal R} G(R) < \epsilon$}
                        \STATE $\bm\beta^{*(k, k^\prime)}, \beta_0^{*(k, k^\prime)}\leftarrow\bm\beta,\beta_0$
                        \BREAK
                    \ENDIF
                    \STATE Remove inactive patterns from $\mathcal A$ using multi safe screening
                    \IF{$m \ge M$}
                        \STATE $\mathcal R \leftarrow \{\argmin_{R\in\mathcal R} G(R)\}$
                    \ENDIF
                \ENDFOR
            \ENDFOR
        \ENDFOR
    \end{algorithmic}
    \caption{Pathwise optimization with multi-reference SPP}
    \label{alg:pathwise_multi_spp}
\end{algorithm}

\begin{algorithm}[p]
    \begin{algorithmic}
        \REQUIRE{$X, \bm y, \{\mathcal I^{(k)}\}_{k\in[K]}, \{\lambda^{(k^\prime)}\}_{k^\prime\in[K^\prime]}, \kappa, \epsilon, M$}
        \ENSURE{$\{R^{(k, k^\prime)}\}_{(k, k^\prime)\in[K]\times[K^\prime]}$}
        \FOR{$k\in[K]$}
            \STATE $X^\prime \leftarrow (X_{\mathcal I^{(k)}}^\top)^\top$
            \STATE $\bm y^\prime \leftarrow \bm y_{\mathcal I^{(k)}}$
            \FOR{$k^\prime\in[K^\prime]$}
                \STATE $\lambda\leftarrow\lambda^{(k^\prime)}$
                \STATE $\mathcal R\leftarrow\emptyset$
                \IF{$k > 1$}
                    \STATE $\mathcal R\leftarrow\mathcal R\cup\{R^{(1, k^\prime)}\}$
                \ENDIF
                \IF{$k^\prime > 1$}
                    \STATE $\mathcal R\leftarrow\mathcal R\cup\{R^{(k, k^\prime-1)}\}$
                \ENDIF
                \IF{$\mathcal R = \emptyset$}
                    \STATE $\mathcal R\leftarrow\mathcal R\cup\{(\bm 0, 0, \bm 0)\}$
                \ENDIF
                \FOR{$R\in\mathcal R$}
                    \STATE $\bm\beta, \beta_0, \bm\alpha\leftarrow R$
                    \STATE Update $\bm\alpha$ from $\bm\beta, \beta_0$ by dual scaling
                \ENDFOR
                \STATE $\mathcal A \leftarrow\mathrm{MultiSafePatternPruning}(X^\prime, \bm y^\prime, \lambda, \kappa, \mathcal R)$
                \FOR{$m\in\{1,2,\ldots\}$}
                    \FOR{$R\in\mathcal R$}
                        \STATE $\bm\beta, \beta_0, \bm\alpha\leftarrow R$
                        \STATE Update $\bm\beta, \beta_0$ of $R$ using the sub-gradient of $P$
                        \STATE Update $\bm\alpha$ from $\bm\beta, \beta_0$ using the dual scaling
                        \STATE $R\leftarrow(\bm\beta, \beta_0, \bm\alpha)$
                    \ENDFOR
                    \IF{$\min_{R\in \mathcal R} G(R) < \epsilon$}
                        \STATE $\bm\beta^{*(k, k^\prime)}, \beta_0^{*(k, k^\prime)}\leftarrow\bm\beta,\beta_0$
                        \BREAK
                    \ENDIF
                    \STATE Remove inactive patterns from $\mathcal A$ using multi safe screening
                    \IF{$m \ge M$}
                        \STATE $\mathcal R \leftarrow \{\argmin_{R\in\mathcal R} G(R)\}$
                    \ENDIF
                \ENDFOR
            \ENDFOR
        \ENDFOR
    \end{algorithmic}
    \caption{Pathwise optimization in cross-validation with multi-reference SPP}
    \label{alg:pathwise_cv_spp}
\end{algorithm}

\end{document}